\documentclass{article}

\PassOptionsToPackage{numbers}{natbib}

\usepackage[preprint]{neurips_2026}

\usepackage[utf8]{inputenc}
\usepackage[T1]{fontenc}

\usepackage[table]{xcolor}
\usepackage{graphicx}
\usepackage{subcaption}
\usepackage{booktabs}
\usepackage{multirow}
\usepackage{caption}
\usepackage{hhline}
\usepackage{tikz}
\usetikzlibrary{calc}
\usepackage{arydshln}
\usepackage[inline]{enumitem}
\usepackage{nicefrac}
\usepackage{comment}
\usepackage{sectsty}
\usepackage{adjustbox}
\usepackage{import}
\usepackage[textsize=tiny]{todonotes}
\usepackage{wrapfig}
\usepackage{float}
\usepackage{placeins}
\usepackage[
	activate={true,nocompatibility},
	tracking=true,
	kerning=true,
	spacing={true},
	factor=1100,
	stretch=10,
	shrink=10,
	nopatch=footnote,
	final,
]{microtype}
\usepackage{url}
\usepackage{longtable}
\usepackage{booktabs}
\usepackage{multirow}
\usepackage{rotating}
\usepackage{pdflscape}

\usepackage{amsmath,amssymb,amsfonts,amsthm,mathtools}

\theoremstyle{plain}
\newtheorem{theorem}{Theorem}[section]

\newtheorem{lemma}[theorem]{Lemma}
\newtheorem{corollary}[theorem]{Corollary}

\theoremstyle{definition}

\theoremstyle{remark}

\makeatletter
\let\latex@underline\underline
\let\latex@underbar\underbar
\makeatother

\usepackage{algorithm}
\usepackage{algpseudocode}
\usepackage{setspace}
\makeatletter
\let\underline\latex@underline
\let\underbar\latex@underbar
\makeatother

\usepackage[colorlinks,citecolor=blue,linkcolor=blue,urlcolor=blue]{hyperref}
\usepackage[capitalize,noabbrev]{cleveref}

\newcommand{\ie}{\textsl{i.\,e.}}
\newcommand{\eg}{\textsl{e.\,g.}}
\newcommand{\defeq}{\coloneqq}

\DeclareMathOperator*{\argmin}{arg\,min}

\newcommand{\E}{\mathbb{E}}
\newcommand{\R}{\mathbb{R}}
\newcommand{\N}{\mathcal{N}}

\newcommand{\In}{\mI_{n}}
\newcommand{\Ik}{\mI_{k}}

\newcommand{\zk}{\begin{pmatrix} \vz, & 0 \end{pmatrix}^\top }
\newcommand{\zkt}{(\vz^{k},\,0)}

\newcommand{\zdk}{\bm{z}_{\downarrow k}}

\newcommand{\Gk}{\bm{G_{k}}}
\newcommand{\Gn}{\bm{G_{n}}}

\newcommand{\Sk}{\bm{\Sigma_{k}}}

\newcommand{\diag}{\text{diag}}
\newcommand{\Tr}{\text{Tr}}

\newcommand{\zmapk}{\hat{\vz}_{\gamma}(k)}
\newcommand{\zmlek}{\hat{\vz}(k)}

\newcommand{\xmapk}{\hat{\vx}_{\gamma}(k)}
\newcommand{\xmlek}{\hat{\vx} (k)}
\newcommand{\Skmat}{\mS_{k}}
\newcommand{\Snkmat}{\mS_{k}^{\perp}}

\newcommand{\zmapkk}{\begin{pmatrix} \zmapk, & 0 \end{pmatrix}^\top }
\newcommand{\zmlekk}{\begin{pmatrix} \zmlek, & 0 \end{pmatrix}^\top }

\newcommand{\Pk}{\mathcal{P}_{k}}
\newcommand{\Pko}{\mathcal{P}_{k}^{\perp}}

\usepackage{amsmath,amsfonts,bm}

\def\eqref#1{equation~\ref{#1}}

\def\1{\bm{1}}

\def\beps{\bm{\epsilon}}

\def\rmA{{\mathbf{A}}}

\def\vx{{\bm{x}}}
\def\vy{{\bm{y}}}
\def\vz{{\bm{z}}}

\def\mA{{\bm{A}}}

\def\mG{{\bm{G}}}

\def\mI{{\bm{I}}}

\def\mM{{\bm{M}}}

\def\mS{{\bm{S}}}

\def\mU{{\bm{U}}}
\def\mV{{\bm{V}}}

\def\mY{{\bm{Y}}}

\DeclareMathAlphabet{\mathsfit}{\encodingdefault}{\sfdefault}{m}{sl}
\SetMathAlphabet{\mathsfit}{bold}{\encodingdefault}{\sfdefault}{bx}{n}

\def\gA{{\mathcal{A}}}

\def\gC{{\mathcal{C}}}
\def\gD{{\mathcal{D}}}
\def\gE{{\mathcal{E}}}

\def\gL{{\mathcal{L}}}

\def\gN{{\mathcal{N}}}

\newcommand{\pdata}{p_{\text{data}}}

\title{Tunable Latent Generative Priors for
Compressed Sensing and Inverse Problems}

\author{
  Sean Gunn \\
  Northeastern University \\
  \texttt{gunn.s@northeastern.edu} \\
  \And
  Jorio Cocola \\
  Harvard University \\
  \texttt{jcocola@seas.harvard.edu} \\
  \And
  Oliver De Candido \\
  Technical University of Munich \\
  \texttt{oliver.de-candido@tum.de} \\
  \AND
  Vaggos Chatziafratis \\
  UC Santa Cruz \\
  \texttt{vaggos@ucsc.edu} \\
  \And
  Paul Hand \\
  Northeastern University \\
  \texttt{p.hand@northeastern.edu} \\
}

\begin{document}

\maketitle

\begin{abstract}
Latent generative models have emerged as powerful priors for solving inverse problems. These models typically represent a class of natural signals at a single, fixed complexity, governed by the latent dimensionality. This can be limiting: depending on the problem, a latent dimensionality that is too small may result in high representation error, while one that is too large may overfit to noise. We develop tunable latent priors for diffusion models, normalizing flows, and variational autoencoders, leveraging nested dropout. Across tasks including compressed sensing, inpainting, denoising, and phase retrieval, we show empirically that tunable priors consistently achieve lower reconstruction errors than fixed-complexity baselines. In the linear denoising setting, we derive the optimal complexity in closed form, showing how it depends on the noise level and the signal spectrum. This work demonstrates the potential of tunable latent generative priors and motivates both the development of supporting theory and their application across a wide range of inverse problems.
\end{abstract}

\section{Introduction}
\label{sec:intro}

Inverse problems aim to reconstruct an unknown signal, potentially corrupted by noise, from a set of measurements given by a forward model, which may or may not be known in advance. Such a formulation applies to various image-processing applications, including compressed sensing, denoising, and super-resolution. In practice, inverse problems are ill-posed and thus they require prior information about the signal to yield a successful recovery~\cite{10035380}.

Deep generative models have been demonstrated to be powerful signal priors when used to solve inverse problems~\cite{bora2017compressed,NEURIPS2018_1bc2029a,asim2020invertible,ILO2021,pmlr-v162-daras22a,song2022solving}.  The process of using a deep generative prior typically has two phases: training and inversion. In the training phase, a generative neural network is trained on a dataset representative of the natural signal class intended for inversion.  In the second phase, the model parameters are fixed from training, and an algorithm is deployed to estimate the signal of interest for a given forward operator. This approach has the benefit that the prior can be learned in isolation from the inverse problems being solved.  Thus, the approach can apply to a variety of inverse problems, which is in contrast to other neural network-based approaches such as end-to-end training.

\begin{figure*}[t!]
	\centering
	\includegraphics[width=1\textwidth]{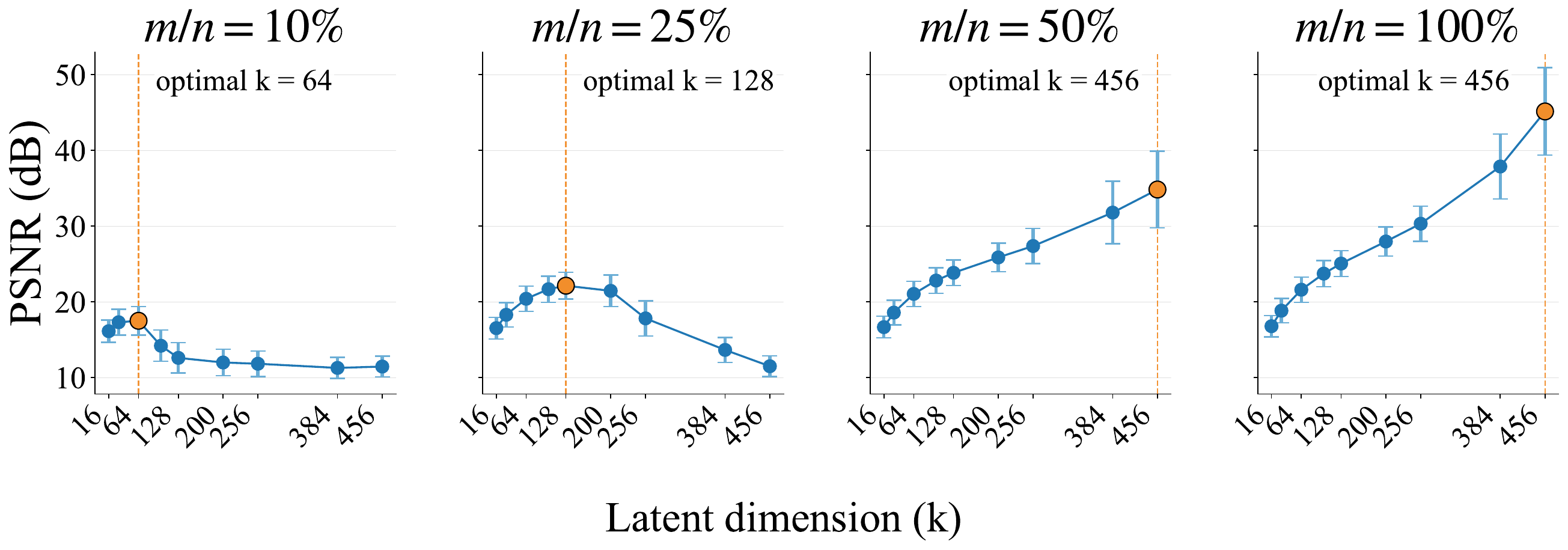}
	\caption{\textbf{Intermediate latent dimensionalities yield the best reconstruction
			at low measurement ratios.} Separate injective flow models (no parameter
		sharing) are trained for each latent dimensionality $k \in [16, 456]$ on
		MNIST ($n = 32 \times 32 = 1024$). Each panel corresponds to a different
		measurement count $m$: for $m < n$, the forward operator is an $m \times n$
		random Gaussian matrix (compressed sensing); note $m = 1024$ is the identity.
		Intermediate $k$ ($64 \leq k \leq 200$) minimizes reconstruction error at
		small $m/n$, with the optimal $k$ shifting as $m$ grows (error bars:
		$\pm 1$ std).}
	\label{fig:naive_tunable}
\end{figure*}

In recent years, there has been a dominant framework for using generative models as priors for inverse problems. Within this framework, generative priors have a fixed complexity that is set during training. For example, existing priors include generative adversarial networks (GANs), which typically have a latent space of fixed low dimensionality \cite{bora2017compressed,NEURIPS2018_1bc2029a,DBLP:conf/cvpr/MenonDHRR20,NEURIPS2020_ad62cfd3, NEURIPS2019_598a9000}; normalizing flow models, which have a latent space of fixed high-dimensionality equal to that of the images \cite{asim2020invertible,whang2021solving,ardizzone2018analyzing,liu2023optimization}; and diffusion models, which maintain information about the probability density over the fixed high-dimensional space of all images \cite{song2022solving,meng2023quantized,jalal2021instance,dou2024diffusion,chung2023diffusion}.

We step outside the dominant framework and demonstrate the benefits of solving inverse problems using generative priors whose complexity can be selected by the user at inference time, after training. Such generative priors simultaneously maintain representations of varying complexities of the natural signal class, and we refer to them as \textbf{tunable generative priors}. \textit{In this paper, we demonstrate that latent tunable generative priors can significantly improve reconstruction when the latent dimension $k$ is
	matched to a specific inverse problem. By training the model to preserve
	meaningful representations across values of $k$, we enable practitioners
	to select an appropriate $k$ at inference time, without retraining.}

To motivate our approach, we first consider injective flows \cite{ross2021tractable}. We train a family of models, each with a distinct latent dimensionality $k$ and no parameter sharing, and apply them to compressed sensing with random measurements. As shown in \Cref{fig:naive_tunable}, reconstruction error follows an inverted U-shaped curve: with fewer measurements, intermediate values of $k$ outperform both larger and smaller ones. Moreover, the optimal complexity depends on the number of available measurements. While this experiment requires training separate models for each $k$, a na\"ive and computationally expensive strategy, it illustrates the importance of tunability. This motivates the more efficient algorithms we develop next, which scale to practical settings.

As an initial theoretical exploration into the effect of model tunability, we study denoising in the case of an invertible linear generative model and the best lower-dimensional linear models that approximate it. In this setting, we derive an explicit maximum-likelihood criterion for selecting the optimal model complexity. Our analysis shows that the optimal complexity depends on the generator’s singular values relative to the noise level: directions with singular values above the noise level improve reconstruction, while lower-signal directions can overfit noise.

Building on this motivation, we will show how tunability can be achieved efficiently in three major classes of generative models: variational autoencoders (VAEs)~\cite{Kingma2014}, normalizing flows (NFs)~\cite{durkan2019neural}, and latent diffusion models (LDMs)~\cite{Rombach_2022_CVPR}. In each case, we design a single family of models with parameter sharing and provide an efficient training procedure. For LDMs, we introduce a new algorithm based on nested dropout~\cite{rippel2014learning}, described in \cref{subsec:LDM_ND}. For NFs, we adopt an existing ordering method~\cite{bekasov2020ordering} (\Cref{app:NF}). For VAEs, we extend the adversarial objective of~\cite{Esser_2021_CVPR} with a nested dropout regularization term. Across all three settings, we empirically demonstrate that tunable generative priors consistently achieve lower reconstruction errors than fixed-complexity baselines over a range of inverse problems, undersampling ratios, and noise levels (\Cref{fig:tunable_CELBA}).

The results in this paper demonstrate the benefits of using generative priors with a tunable complexity for inverse problems both empirically and theoretically. Empirically, we show that tunable complexity consistently improves reconstruction quality across multiple generative model architectures, inverse problem types, and inversion algorithms.

\textbf{Our main contributions are as follows:}
\begin{itemize}

	\item We observe a potentially surprising phenomenon in the use of latent generative priors for solving inverse problems. We train a single generative model to simultaneously represent the natural signal class across multiple latent dimensionalities $k$. Across a variety of architectures, inverse problems, and inversion algorithms, we empirically find that some intermediate latent dimensionality yields an improved reconstruction error.

	\item We propose a new training algorithm for latent diffusion that utilizes nested dropout and a convex combination of the original and truncated latent objectives. This yields a tunable latent diffusion model that learns hierarchical representations across latent dimensionalities, enabling a single model to be tuned as a prior for downstream inverse problems.

\end{itemize}
\section{Background and Related Work}
The background will focus on diffusion models due to their relevance in recent research trends. For additional details on normalizing flows and variational autoencoders, we refer the reader to the appropriate references.

\subsection{Diffusion Models}
\label{sub:DM}

The standard diffusion framework gradually corrupts data
$\vx_0\!\sim\!\pdata\ \vx\in\R^n$ into noise $\vx_T\!\sim\!\gN(0,\mI_n)$ via a forward noising process, and learns a parameterization of the reverse process to recover the data distribution~\citep{NEURIPS2020_4c5bcfec,song2021scorebased,dhariwal2021diffusion}. We adopt latent diffusion models (LDMs)~\citep{Rombach_2022_CVPR}, which improve computational efficiency by performing diffusion in a learned latent space: a \emph{variational autoencoder (VAE)}~\citep{Kingma2014} with encoder $\gE\colon\R^n\!\to\!\R^k$ and decoder $\gD\colon\R^k\!\to\!\R^n$ (typically $k\!\ll\!n$) yields $\vz_0=\gE(\vx_0)$ and $\vx_0\!\approx\!\gD(\vz_0)$ for all $\vx_0\!\sim\!\pdata$. The encoder maps $\pdata$ to the latent distribution $p_Z\defeq \gE_{\#}\pdata$, and the diffusion process then operates on the latent variable $\vz_t$.

With schedule $\{\beta_t\}_{t=1}^T$, $\alpha_t\defeq 1-\beta_t$, and $\bar\alpha_t\defeq \prod_{j=1}^t\alpha_j$ (with $\bar\alpha_0=1$), the forward marginal in latent space is
\begin{align}
	\label{eq:ddpm_marginal}
	\vz_t = \sqrt{\bar\alpha_t}\,\vz_0 + \sqrt{1-\bar\alpha_t}\,\beps,
	\qquad \beps\sim\gN(0,\mI_k), \tag{forward marginal}
\end{align}
which allows sampling $\vz_t$ at any timestep $t$ directly from $\vz_0$. The reverse process is parameterized by a neural network:
\begin{align}
	\label{eq:ddpm_reverse}
	\vz_{t-1} = \tfrac{1}{\sqrt{\alpha_t}}\!\left(\vz_t -
	\tfrac{1-\alpha_t}{\sqrt{1-\bar\alpha_t}}\,\epsilon_\theta(\vz_t,t)\right)
	+ \sigma^{\text{DDPM}}_t\,\beps'_t,
	\qquad \beps'_t\sim\gN(0,\mI_k), \tag{reverse}
\end{align}
where $\displaystyle \sigma^{\text{DDPM}}_t=\sqrt{\tfrac{1-\bar\alpha_{t-1}}{1-\bar\alpha_t}\,(1-\alpha_t)}$ is the posterior standard deviation. Write $\hat{\beps}_t=\beps_\theta(\vz_t,t)$ for the predicted noise and define the clean-latent estimate
\[
	\hat{\vz}_0(\vz_t)=\frac{\vz_t-\sqrt{1-\bar\alpha_t}\,\hat{\beps}_t}{\sqrt{\bar\alpha_t}}.
\]
DDIM~\citep{song2021denoising} uses a distinct reverse update:
\[
	\vz'_{t-1}
	=\sqrt{\bar\alpha_{t-1}}\,\hat{\vz}_0(\vz_t)
	+\sqrt{1-\bar\alpha_{t-1}-\sigma_t^2}\,\hat{\beps}_t
	+\sigma_t\beps'_t,
	\qquad
	\sigma_t=\eta\sqrt{\frac{1-\bar\alpha_{t-1}}{1-\bar\alpha_t}(1-\alpha_t)}.
\]
Here $\eta\in[0,1]$ controls DDIM stochasticity, and $\eta=0$ gives deterministic DDIM.

A denoising network $\epsilon_{\theta}$ (typically a U-Net) is trained to predict the noise in the latent space:
\begin{align}
	\label{eqn:ldm_standard}
	\mathcal{L}_{\text{LDM}}
	= \E_{\,t\sim\mathcal{U}(\{1,\dots,T\}),\ \vz_0,\ \beps\sim\gN(0,\mI)}
	\big\|\,\beps - \epsilon_{\theta}\!\big(\sqrt{\bar\alpha_t}\,\vz_0 + \sqrt{1-\bar\alpha_t}\,\beps,\,t\big)\big\|_2^2,
\end{align}
where $t \sim \mathcal{U}(\{1,\dots,T\})$ denotes that the timestep is sampled uniformly from all diffusion steps.
To enable latent diffusion, the autoencoder is trained using reconstruction, latent regularization, and adversarial losses. These terms respectively encourage accurate reconstructions, alignment with a reference latent prior, and perceptually realistic outputs:
\begin{align}
	\label{eq:auto_train}
	\gL_{\text{Autoencoder}}
	= \gL_{\text{rec}}\big(\vx,\,\gD(\gE(\vx))\big)
	+ \lambda_{\text{reg}}\,\gL_{\text{reg}}\big(\gE(\vx)\big)
	+ \lambda_{\text{adv}}\,\gL_{\text{adv}}\!\big(\gC,\,\vx,\,\gE(\vx)\big).
\end{align}

\subsection{Inverse Problems with Generative Priors}
\label{sub:IP_Background}
We consider the general linear inverse problem of recovering an unknown signal
$\vx \in \R^n$ from noisy measurements $\vy \in \R^m$,
\begin{align}
	\vy = \gA(\vx) + \bm{\eta}, \quad \bm{\eta} \sim \gN(0,\sigma^{2}\mI_{m}),
\end{align}
where $\gA\colon\R^n \to \R^m$ \citep{10035380} is the forward operator (\eg, a linear projection, convolutional blur, or other transformation), and $\bm{\eta}$ is additive noise drawn from a Gaussian distribution. Generative models can be employed as priors for inverse problems in two main ways:
supervised, where the forward operator is known during training and models
are trained on paired data $\{(x_i, y_i)\}$, or unsupervised, where the
forward operator is unknown~\citep{10035380,li2025diffusionmodelsimagerestoration}.

For inverse problems, we modify the reverse dynamics in \cref{eq:ddpm_reverse}
to account for measurements $\vy = \gA(\vx) + \bm{\eta}$.
Since we operate entirely in latent space ($\vx=\gD(\vz)$), posterior sampling is guided by the
conditional score, which decomposes as
\begin{align}
	\nabla_{\vz_t} \log p_t(\vz_t \mid \vy)
	= \nabla_{\vz_t} \log p_t(\vz_t)
	+ \nabla_{\vz_t} \log p_t(\vy \mid \vz_t).
\end{align}
The first term is provided by the pretrained diffusion prior through the
denoising network $\epsilon_\theta$, while the second term enforces
data-consistency with the forward operator $\gA$ under the measurement model.

Following the taxonomy of interleaving methods highlighted by \cite{wang2024dmplug}, these solvers alternate unconditional reverse diffusion steps (\eg, DDIM) with data-consistency corrections that either explicitly approximate the measurement likelihood or enforce feasibility. One line of methods explicitly approximates the measurement-likelihood term via a projection or gradient update~\citep{jalal2021instance,kawar2021snips,chung2023diffusion,wang2023zeroshot,RoutCKCSC24}. A canonical formulation is latent diffusion posterior sampling~\citep{rout2023solving}:
\begin{align}
	\nabla_{\vz_t}\log p(\vy\mid\vz_t)
	\;\approx\;
	-\frac{1}{2\sigma^2}\nabla_{\vz_t}
	\left\|\vy-\gA\!\left(\gD(\hat{\vz}_0(\vz_t))\right)\right\|_2^2.
\end{align}
This is a Gaussian-likelihood approximation for measurement noise $\sigma>0$, evaluated at the noise-prediction estimate $\hat{\vz}_0(\vz_t)$ defined in \cref{sub:DM}. This estimate approximates the posterior mean via Tweedie's formula~\citep{Efron01122011}. Differentiation is through the full map $\vz_t\mapsto\hat{\vz}_0(\vz_t)\mapsto\gD(\hat{\vz}_0(\vz_t))\mapsto\gA(\gD(\hat{\vz}_0(\vz_t)))$, including the denoising network.

Another class of methods does not explicitly compute the measurement likelihood, but instead approximates the feasible set of solutions $\{\vx \mid \vy=\gA(\vx)\}$~\cite{kawar2021snips,kawar2022denoising,song2024solving}. In particular \cite{song2024solving} employs  a hard data-consistency term that enforces $\gD(\vz)\in\{\,\vx \mid \vy=\gA(\vx)\,\}$, which is approximated in practice by gradient descent.

\section{Methods}

\begin{figure*}[t!]
	\centering
	\captionsetup{hypcap=false}
	\includegraphics[height=0.175\textheight]{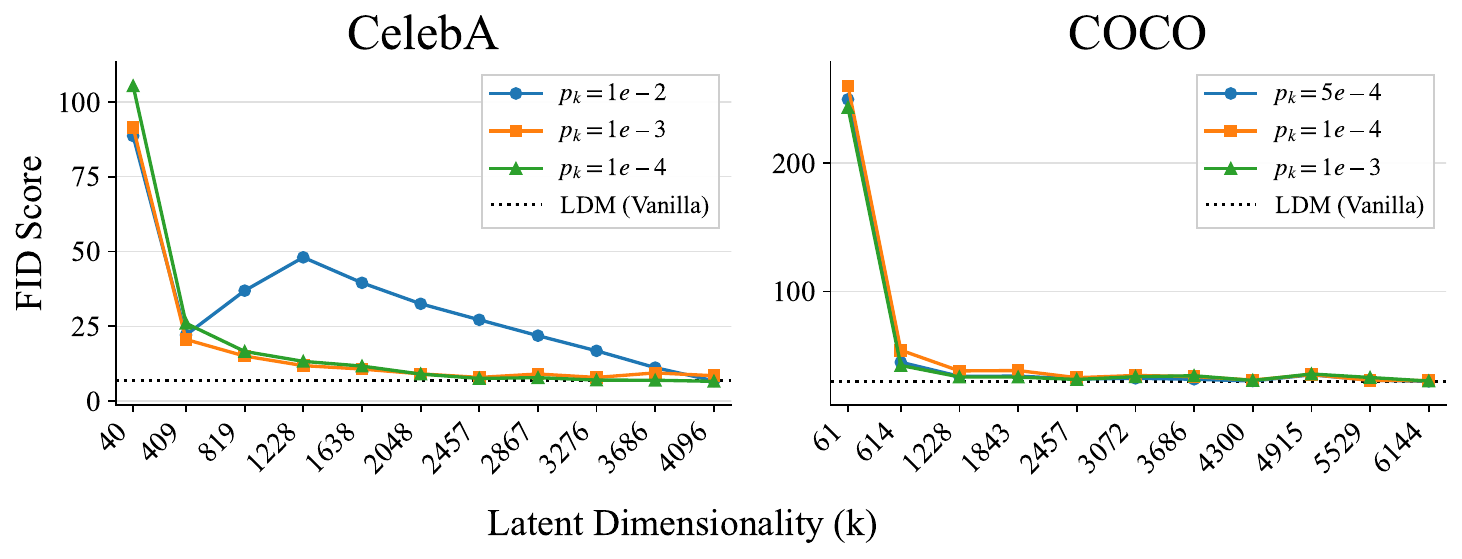}
	\caption{\textbf{Nested dropout training produces tunable latent diffusion models that maintain generation quality across latent dimensionalities.} FID score is plotted as a function of latent dimensionality $k$ for models trained with different values of the dropout distribution parameter $p_k$. The vanilla LDM baseline (dotted line) operates only at full dimensionality. As $k$ increases, the tunable models approach baseline performance while retaining the ability to generate from lower-dimensional representations. Results are evaluated on 50k training images from each dataset. FID scores computed with \citep{Parmar_2022_CVPR}.}
	\label{fig:fid_both}
\end{figure*}

This section describes the training of a tunable latent diffusion model (\cref{subsec:LDM_ND}) and its application as a prior for inverse problems (\cref{subsec:IP_TP}). We begin by training a variational autoencoder (VAE) as the backbone of the latent diffusion model. The VAE consists of an encoder $\gE\colon\R^n \to \R^k$ and a decoder $\gD\colon\R^k \to \R^n$, such that for samples $\vx_0 \sim \pdata(\vx_0)$ we have $\vx_0 \approx \gD(\gE(\vx_0))$. Once the autoencoder is trained, a denoising network $\beps_\theta(\vz_t,t)$ is trained in latent space to predict the Gaussian noise at each diffusion step $t$. Finally, diffusion is performed in latent space using the denoising network $\beps_\theta$.

\subsection{Training a Tunable Latent Diffusion Model }
\label{subsec:LDM_ND}
We aim to train a latent generative model that can represent the natural signal class across multiple latent dimensionalities. To achieve this, we leverage nested dropout~\citep{rippel2014learning}, which imposes an ordered structure on the latent variables by always preserving a prefix of coordinates. Formally, let $k \sim p_k$ be drawn from $\{1,\dots,d\}$.
In our experiments, we use a truncated geometric distribution with success parameter $p$, but other distributions over $\{1,\dots,d\}$ can be adopted depending on the application. The truncation operator is defined as
$\zdk = [\vz_{1},\vz_{2},\dots,\vz_{k},0,\dots,0]$ with $k \leq d$.
Given $\vz=\gE(\vx)$, reconstruction is performed from $\zdk$,
which encourages earlier coordinates to carry more information about the signal class.

Latent diffusion is trained in two stages. First, the autoencoder is trained to reconstruct $\vx\in\R^n$ robustly across a range of latent dimensions $k$. Building on \cref{sub:DM}, the VAE backbone is trained with reconstruction, regularization, and adversarial terms, which we extend here with a nested dropout objective. This follows the standard VAE training objective with an added nested dropout term:
\begin{align}
	\label{eq:vae_drop}
	\mathcal{L}_{\text{VAE}}
	= \min_{E,\,D}\ \max_{\gC}\ \Big[
		                            &\ \gL_{\text{rec}}\big(\vx,\,\gD(\gE(\vx))\big)
		                            + \lambda_{\text{reg}}\,\gL_{\text{reg}}\big(\gE(\vx)\big) \nonumber \\
		                            &\ + \lambda_{\text{adv}}\,\gL_{\text{adv}}\!\big(\gC,\,\vx,\,\gE(\vx)\big)
		                            + \lambda_{\text{drop}}\,\gL_{\text{drop}}\big(\vx,\,\gD(\gE(\vx)_{\downarrow k})\big)
		                            \Big],
\end{align}
where $\gC : \R^{n} \to (0,1)$ is a discriminator.
We adopt a perceptual loss~\citep{zhang2018unreasonableeffectivenessdeepfeatures} for $\gL_{\text{drop}}$.

In the second stage, the diffusion model is trained in latent space with a loss that interpolates  between the standard diffusion objective and its truncated-latent variant. For $\lambda \in [0,1]$, we define
\begin{align}
	\gL_{\text{LDM}}
	 & = \E_{\gE(\vx),\beps \sim \mathcal{N}(0,1), t \sim \mathcal{U}(\{1,\dots,T\})}
	\Big[(1-\lambda)
		\left\| \beps - \beps_{\theta}(\vz_t, t) \right\|_2^2
		+ \lambda \left\| \beps - \beps_{\theta}((\vz_t)_{\downarrow k}, t) \right\|_2^2 \Big].
\end{align}

The first term is the standard latent diffusion objective, while the second applies it to the truncated latent $(\vz_t)_{\downarrow k}$, encouraging effective denoising even from a reduced representation. Here, $\vz_t$ is sampled from the forward DDPM process (\cref{eq:ddpm_marginal}), and $k$ is drawn from the same distribution $p_k$ used in the VAE objective, ensuring consistency. The result is a hierarchically organized latent space: the first few coordinates capture the most essential signal structure, with each additional coordinate contributing finer detail as the model's representational capacity grows with $k$. This effect is illustrated in \cref{fig:fid_both}, which shows how FID varies with latent dimensionality under different hyperparameters. Lastly, for details about training, please see Appendix~\ref{app:LDM}.

\subsection{Inverse Problems with Tunable Priors}
\label{subsec:IP_TP}
\begin{algorithm}[t!]
	\caption{General Template for Tunable Diffusion Priors}
	\label{algo:general}
	\setstretch{0.85}
	\begin{algorithmic}[1]
		\State \textbf{Input:} $\vy$, $\gA$, $\gD$, $\beps_\theta$, steps $T$, schedule $\{\bar\alpha_t\}_{t=0}^T$, parameter $k$
		\Statex \hspace{\algorithmicindent} Reverse sampler (DDIM or DDPM), data-consistency correction
		\State \textbf{Output:} $\gD(\vz_0)$
		\State $\vz_T\sim\mathcal{N}(0,\mI_d)$; $\vz_T\gets(\vz_T)_{\downarrow k}$
		\For{$t=T,\ldots,1$}
		\State $\hat{\beps}_t\gets\beps_\theta(\vz_t,t)$
		\State $\hat{\vz}_0\gets(\vz_t-\sqrt{1-\bar\alpha_t}\,\hat{\beps}_t)/\sqrt{\bar\alpha_t}$
		\State $\vz'_{t-1}\gets$ chosen DDIM or DDPM reverse update using $\vz_t,\hat{\vz}_0,\hat{\beps}_t$
		\State $\vz_{t-1}\gets$ data-consistency correction of $\vz'_{t-1}$ using $\vy,\gA,\gD$
		\State $\vz_{t-1}\gets(\vz_{t-1})_{\downarrow k}$
		\EndFor
		\State \textbf{return} $\gD(\vz_0)$
	\end{algorithmic}
\end{algorithm}

A broad family of methods applies a data-consistency step via projection, gradient, or small optimization after each reverse update to move the prior iterate toward the feasible set $\{\vx \mid \gA(\vx)=\vy\}$. In latent space, representative examples include Posterior Sampling with Latent Diffusion (PSLD)~\citep{rout2023solving}, ReSample~\citep{song2024solving}, DAPS~\citep{zhang2025improving}, and our formulation in \cref{algo:posterior}. \Cref{algo:posterior} is a simplified mathematical template for a broadly applicable instantiation across a variety of inverse problems, including nonlinear forward operators. The experiments in \cref{fig:tunable_CELBA} employ an implementation of this template, while \cref{algo:general} describes how the same tunability principle can be incorporated into other latent-space inversion methods.

Both algorithms use $t=T,\ldots,1$ with $\bar\alpha_0=1$ and full latent dimension $d$. For a subsampled schedule, $t-1$ denotes the previous selected timestep (the terminal clean state is indexed by $0$), and the reverse formulas use the effective $\alpha_t=\bar\alpha_t/\bar\alpha_{t-1}$ and $\beta_t=1-\alpha_t$.

Starting from Gaussian noise truncated to $k$ coordinates in the latent space, \cref{algo:general} alternates a learned generative update, a data-consistency correction, and a truncation step. The generative update follows the diffusion prior, the data-consistency step enforces agreement with the measurements $\vy=\gA(\gD(\vz))$, and the truncation operator $(\cdot)_{\downarrow k}$ controls the effective complexity of the representation. Although \cref{algo:general} is written for diffusion priors, the same tunability principle can be applied to flow-based generative models~\citep{lipman2023flow}, where the generative update is defined by a learned velocity field rather than a denoising or score model.
\section{Theory for Denoising with Linear Generative Model}
\begin{figure*}[b!]
	\centering
	\includegraphics[width=0.99\textwidth]{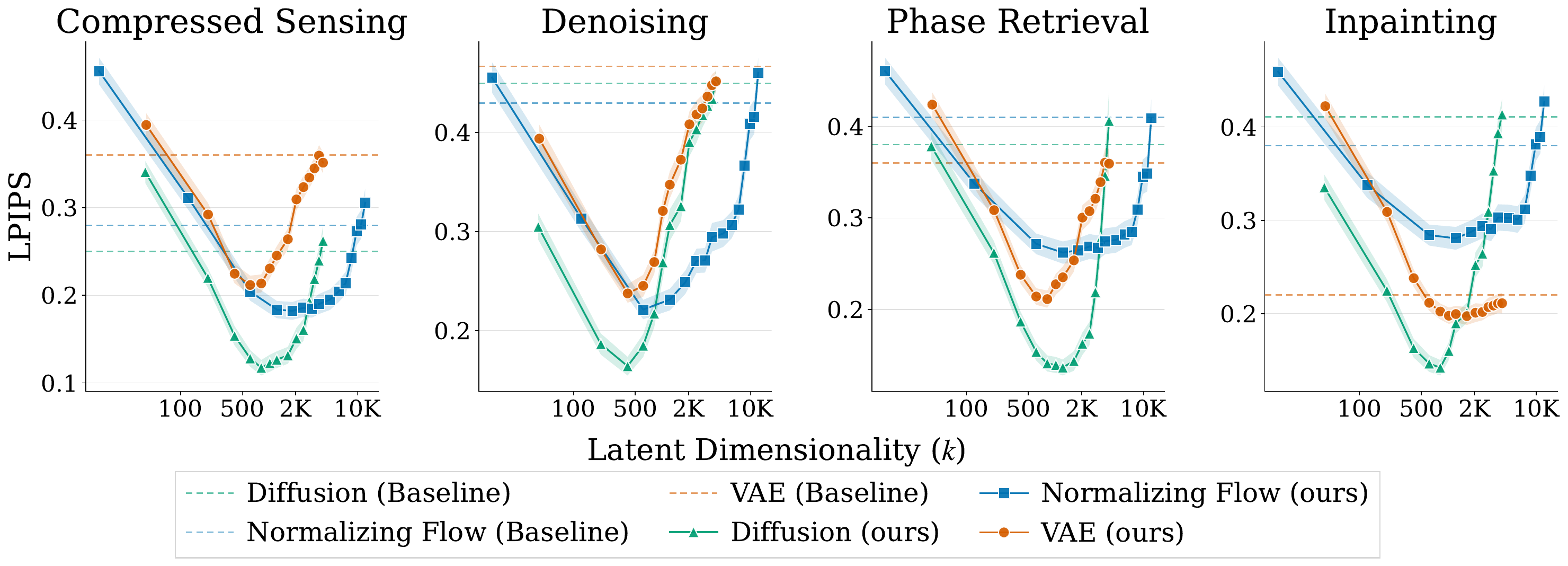}
	\caption{\textbf{Tunable priors outperform fixed-complexity baselines across multiple different methods and inverse problems.} Reconstruction performance (LPIPS, lower is better) is shown as a function of latent dimensionality $k$ for compressed sensing, denoising, phase retrieval, and inpainting on the FFHQ dataset. The tunable LDM prior is compared against the baseline operating at full dimensionality.}
	\label{fig:tunable_CELBA}
\end{figure*}
\label{sec:theory}
The goal of this section is to provide preliminary theoretical conclusions that show tunability can lead to improved reconstruction errors relative to corresponding nontunable models.  Further, this section aims to provide justification for how to select the tuning parameter for a generative model with tunable complexity. We illustrate this in the setting of denoising under additive Gaussian noise, where the benefits of tunability can already be shown for linear generative models.

Consider a family of linear generative models given as follows. Fix an invertible $\mG \in \R^{n \times n}$, and let $\mG=\mU \bm{\Sigma} \mV^{T}$ be a singular value decomposition with $\bm{\Sigma} = \diag(s_1, \ldots, s_n) \in \R^{n \times n}$.
For any $k \leq n$, let $\Gk=\mU \Sk \mV^{T} \in \R^{n \times n}$, where $\Sk = \diag(s_1, \ldots, s_k, 0, \ldots, 0) \in \R^{n \times n}.$  Note that $\mG = \mG_n$.
Each $\mG_k$ induces a probability distribution over $\R^n$ by
\[
	\vx = \mG_k \vz \text{, where }\vz \sim \mathcal{N}(0, \mI_n).
\]
We will refer to this distribution as $p_{\mG_k}$ and observe that $p_{\mG_k} = \N(0, \mG_k \mG_k^T)$.  Note that $k$ is a parameter that governs the complexity of the modeled signal class.

We consider the following denoising problem. Let $\vy= \vx_{0} + \bm{\eta}$, where $\vx_{0} \sim p_{\Gn}$ and $ \bm{\eta} \sim \N(0,\sigma^2 \In)$.  Our goal is to recover $\vx_0$ given $\vy$ and $\mG_n$.   For a given $k$, we consider a maximum likelihood estimate (MLE) of $\vx_{0}$ under the signal prior $p_{\mG_k}$.  This results in the following optimization problem over the latent space:
\begin{align}
	\label{eq:latent_mle}
	\zmlek : \hspace{-.25em}
	 & = \argmin_{\vz^k \in \R^k}\frac{1}{2} \Bigl\|\vy-\Gk \zkt \Bigr\|^2 ,
\end{align}
where $\zkt \in \R^n$ is the vector $\vz^k$ padded with zeroes; the estimated signal is $\xmlek = \mG_k \zmlekk $.  The following theorem provides an exact expression for the mean square error of the estimate above.

\begin{theorem}
	\label{thm:mse_k}
	Suppose we have a family $\{ \Gk \}_{k = 1, \dots, n}$ of generative models as given above, and let $p_{\Gk} = \N(0,\Gk \Gk^T)$, and let $\Gn \in \R^{n \times n}$ have singular values $s_{1} \geq s_{2} \geq \cdots \geq s_{n} >0 $. Let $\vx_{0} \sim p_{\mG_n} $ and $ \bm{\eta} \sim \N(0,\sigma^2 \In)$. Then the estimator given by \eqref{eq:latent_mle} yields
	\begin{align}
		\E_{\vx_0,\,\bm{\eta}}\big[\,\|\xmlek - \vx_{0}\|^2\,\big]
		= k\sigma^2 + \sum\limits_{j=k+1}^{n} s_{j}^2. \nonumber
	\end{align}

\end{theorem}

The exact expression for reconstruction error in Theorem \ref{thm:mse_k} makes it possible to analytically find the optimal value of the tunable complexity parameter $k$, as established in the following corollary.

\begin{corollary}
	\label{cor:low_k}

	Under the assumptions of Theorem~\ref{thm:mse_k},
	suppose additionally that $s_1 \geq \sigma$.
	Then the reconstruction error is minimized by
	\begin{align}
		\argmin_{k \in [n]} \E_{\vx_0, \bm{\eta}} \| \xmlek- \vx_{0} \|^2 = \max \Bigl\{ k \mid  s_k \geq \sigma \Bigr\}. \nonumber
	\end{align}
\end{corollary}

The corollary shows that the optimal model retains only singular directions above the noise level, so higher noise leads to a smaller optimal $k$. Proofs are provided in Appendix \cref{app:proofs}, along with the maximum \textsl{a posteriori} formulation.

\section{Experiments}
Our experiments underscore several key findings.
First, we demonstrate that a single generative model can be trained to operate across a wide range of latent dimensionalities, enabling the same model to be meaningfully tuned at inference time rather than retrained for each complexity level (\cref{fig:fid_both}).
Second, in the denoising setting, we empirically validate the theoretical findings of Section~\ref{sec:theory}, which identify an intermediate optimal latent complexity in the presence of additive noise (\cref{fig:nf_flow_ushape}).
Third, we show that the observed non-monotonic dependence on latent dimensionality generalizes across multiple inverse problems, datasets, and generative architectures, including latent diffusion models, variational autoencoders, and normalizing flows (\cref{fig:tunable_CELBA,fig:nf_flow_ushape,fig:vae_flow_ushape}).
Finally, we demonstrate that incorporating tunability can complement recent advances in inversion algorithms, with especially large gains for methods that explicitly optimize for measurement consistency, such as ReSample~\cite{song2024solving}, flow-based MAP estimation~\cite{NEURIPS2024_698570ae}, and Algorithm~\ref{algo:posterior}, and more mixed effects for gradient-guided posterior-sampling methods~\citep{patel2024flowchef,zhang2025improving,rout2023solving} (\cref{tab:celebahq,tab:ffhq,tab:img_net}).
We include both positive and negative paired comparisons in these broader solver experiments to show where tunability is most useful, rather than to suggest that it should improve every inversion method uniformly. Training, evaluation, inverse problem implementations, configurations, and environment instructions are available in our public \href{https://github.com/g33sean/Latent-Generative-Models-with-Tunable-Complexity-for-Compressed-Sensing-and-other-Inverse-Problems}{GitHub repository}.

\paragraph{Datasets and Metrics.}
We evaluate on five datasets spanning different resolutions: CelebA~\citep{liu2015faceattributes}, CelebA-HQ~\citep{karras2018progressive}, ImageNet~\citep{deng2009imagenet}, and MS COCO~\citep{lin2015microsoftcococommonobjects} at $64 \times 64 \times 3$, and FFHQ~\citep{karras2019style} at $256 \times 256 \times 3$. For generative quality, we report the Fréchet Inception Distance (FID $\downarrow$) computed on 50k samples. For inverse problems, we measure reconstruction fidelity using Peak Signal-to-Noise Ratio (PSNR $\uparrow$) and perceptual quality using LPIPS ($\downarrow$). Unless otherwise specified, reconstruction results are reported on a test set of 100 images. For the experiments reported in \cref{tab:ffhq,tab:img_net}, the latent dimensionality $k$ was selected using a held-out validation set of 12 images (see \ref{app:LDM} for details).

\paragraph{Baselines.}
We compare against representatives of three families of generative priors for inverse problems.
Flow matching: Flow CHEF~\citep{patel2024flowchef} and Flow ICTM~\citep{NEURIPS2024_698570ae}.
Latent diffusion: PSLD~\citep{rout2023solving}, LatentDAPS~\citep{zhang2025improving}, and ReSample~\citep{song2024solving}. Normalizing flows: NF~\citep{asim2020invertible}, which performs MAP estimation in the flow's latent space. For each baseline, we additionally construct a \emph{tunable variant} that replaces the fixed latent prior with a tunable one while leaving the inversion procedure unchanged, isolating the effect of prior flexibility.

\begin{figure}[t!]
	\centering
	\includegraphics[width=1.0\textwidth,keepaspectratio]{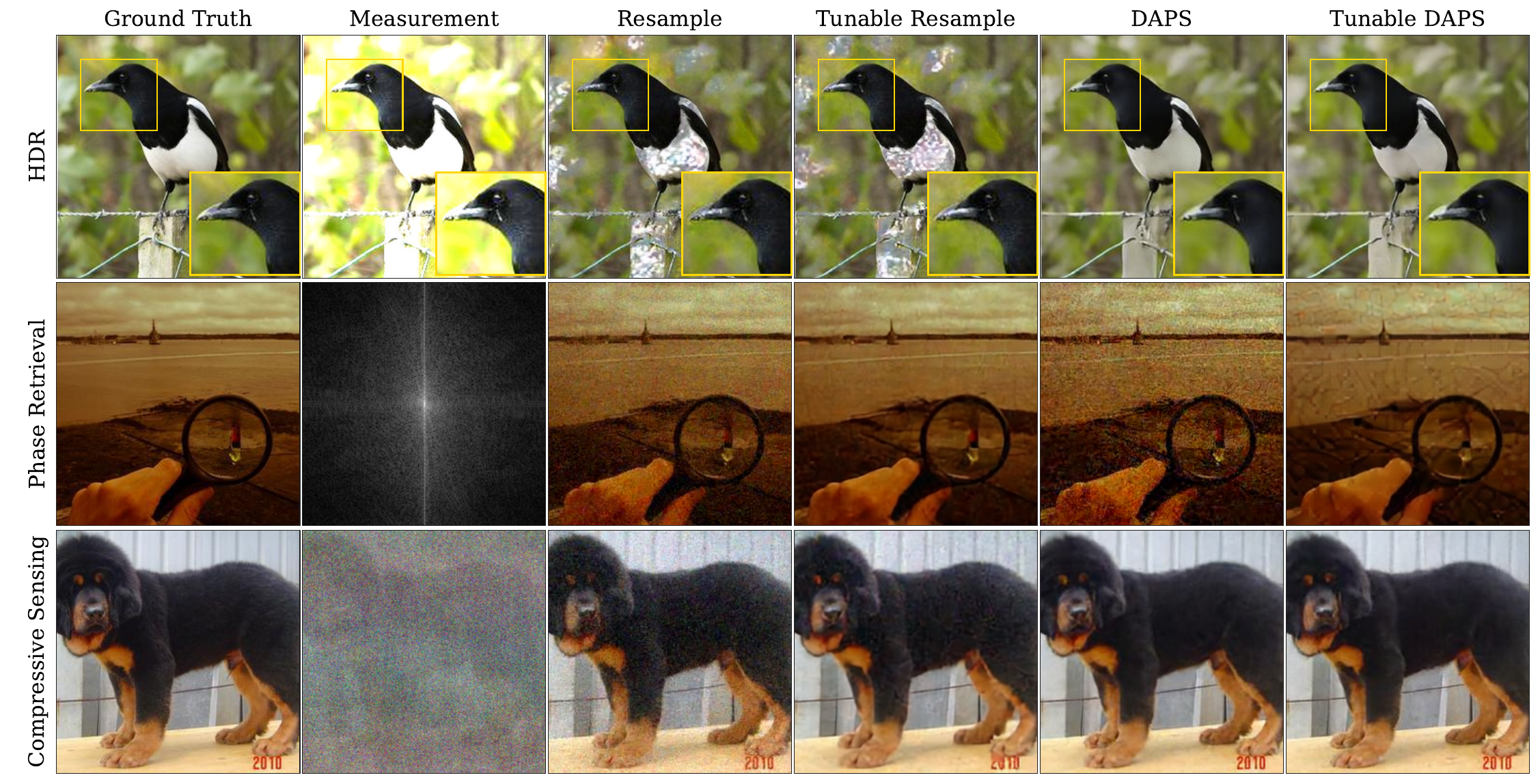}
	\caption{\textbf{Qualitative results on ImageNet across three inverse problems.} Rows: HDR (clipping with scale factor $2$), phase retrieval ($m/n = 12.5\%$), and compressed sensing ($m/n = 10\%$). Columns: ground truth, measurement ($\mathbf{A}^\top \mathbf{y}$ for CS and phase retrieval; clipped input for HDR), and reconstructions from ReSample~\citep{song2024solving} and LatentDAPS~\citep{zhang2025improving} alongside their tunable counterparts. Insets show enlarged views of the regions highlighted in yellow.}
	\label{fig:qualitative_inverse_problems}
\end{figure}

\paragraph{Implementation Details.}
We build upon the Diffusers package~\citep{von-platen-etal-2022-diffusers} for both the denoising U-Net and VAE architectures. All VAE models employ continuous latent spaces without quantized codebooks. For each dataset, the VAE was trained on the same data used for subsequent diffusion model training, with a held-out validation set used to evaluate LPIPS for model selection. For the LDM, we employ a U-Net in half-precision, following the architectural specifications of~\citep{Rombach_2022_CVPR}, and optimize using AdamW~\citep{loshchilov2019decoupled}. During training, we maintained a held-out set to evaluate FID, selecting the model with the lowest FID score for inverse problem applications. All models were trained on a single H200 GPU, and all inverse problems were solved on a single A100 (40\,GB). Additional implementation details are provided in \cref{app:LDM}.

\paragraph{Inverse Problem Details}
For \cref{fig:tunable_CELBA}, we evaluated four forward operators.  All experiments were done on $64 \times 64 \times 3$.
Compressed sensing used a random Gaussian operator $\mA \in \R^{m \times n}$ with i.i.d.\ $\gN(0, 1/m)$ entries, where $m=1228$ (approximately $10\%$ of $n=12288$).
Denoising was modeled as additive Gaussian noise $\bm{\eta} \sim \mathcal{N}(0, \sigma^2 \mI_n)$ with $\sigma = 0.2$.
Inpainting employed a binary mask $\mM$ with approximately $80\%$ missing pixels.
Phase retrieval used phaseless Gaussian measurements $\lvert \mA \vx \rvert$, sharing the same $\mA$ as compressed sensing, with a measurement ratio $m/n = 0.25$ ($25\%$).
All inverse problems were solved using Algorithm~\ref{algo:posterior}.

In \cref{tab:ffhq,tab:img_net}, we compare our approach to a contemporary state-of-the-art (SOTA) diffusion prior and show that our framework can also enhance other SOTA priors on certain inverse problems. We deliberately include settings in which our method does \emph{not} improve reconstruction error, to make clear that it is not intended as a universal one-shot solver.

We consider the following measurement operators: compressed sensing with measurement ratio $m/n = 10\%$; inpainting with $85\%$ of pixels missing uniformly at random; Gaussian deblurring with a $9 \times 9$ kernel of standard deviation $3$; bicubic $4\times$ super-resolution; and HDR, implemented as clipping with scale factor $2$. For phase retrieval, the forward map applies a Rademacher sign flip followed by circular convolution with a randomly drawn Gaussian filter, after which the signal is subsampled to $m$ measurements; we use $m/n = 15\%$ on ImageNet and $m/n = 12.5\%$ on FFHQ. Across all settings, we add Gaussian measurement noise with standard deviation $\sigma = 0.05$.

\begin{table}[t!]
	\centering
	\scriptsize
	\setlength{\tabcolsep}{1.5pt}
	\renewcommand{\arraystretch}{1}
	\newcommand{\na}{\multicolumn{2}{c}{---}}
	\resizebox{1.0\textwidth}{!}{%
		\begin{tabular}{@{}l*{10}{c}@{}}
			\toprule
			\multirow{2}{*}{\textbf{Method}}
			 & \multicolumn{2}{c}{\textbf{Compressed Sensing }}
			 & \multicolumn{2}{c}{\textbf{Inpainting}}
			 & \multicolumn{2}{c}{\textbf{Super Resolution }}
			 & \multicolumn{2}{c}{\textbf{Phase Retrieval}}
			 & \multicolumn{2}{c}{\textbf{HDR}}                                                                                                                                                                         \\
			\cmidrule(lr){2-3}
			\cmidrule(lr){4-5}
			\cmidrule(lr){6-7}
			\cmidrule(lr){8-9}
			\cmidrule(lr){10-11}
			 & PSNR$\uparrow$                                   & LPIPS$\downarrow$
			 & PSNR$\uparrow$                                   & LPIPS$\downarrow$
			 & PSNR$\uparrow$                                   & LPIPS$\downarrow$
			 & PSNR$\uparrow$                                   & LPIPS$\downarrow$
			 & PSNR$\uparrow$                                   & LPIPS$\downarrow$                                                                                                                                     \\
			\midrule
			Flow CHEF~\citep{patel2024flowchef}
			 & \textbf{26.31}                                   & \textbf{0.36}     & 24.53          & 0.42          & 24.07          & 0.52          & \na            & \na                                            \\
			Tunable Flow CHEF (Ours)
			 & 26.24                                            & \textbf{0.36}     & \textbf{24.76} & \textbf{0.39} & \textbf{24.41} & \textbf{0.49} & \na            & \na                                            \\
			\midrule
			Flow ICTM~\citep{NEURIPS2024_698570ae}
			 & 22.98                                            & 0.57              & 21.62          & 0.58          & 20.91          & 0.59          & \na            & \na                                            \\
			Tunable Flow ICTM (Ours)
			 & \textbf{26.13}                                   & \textbf{0.41}     & \textbf{22.60} & \textbf{0.51} & \textbf{21.38} & \textbf{0.56} & \na            & \na                                            \\
			\midrule
			LatentDAPS~\citep{zhang2025improving}
			 & \textbf{28.5}                                    & \textbf{0.27}     & 24.21          & 0.35          & \textbf{26.68} & \textbf{0.27} & \textbf{29.32} & \textbf{0.23} & \textbf{19.79} & \textbf{0.33} \\
			Tunable LatentDAPS (Ours)
			 & 28.31                                            & 0.29              & \textbf{24.50} & \textbf{0.34} & 26.61          & \textbf{0.27} & \textbf{29.32} & \textbf{0.23} & 19.36          & 0.34          \\
			\midrule
			PSLD~\citep{rout2023solving}
			 & 24.92                                            & 0.37              & 17.12          & 0.59          & 21.95          & 0.46          & \na            & \na                                            \\
			Tunable PSLD (Ours)
			 & \textbf{25.41}                                   & \textbf{0.34}     & \textbf{17.55} & \textbf{0.57} & \textbf{22.36} & \textbf{0.43} & \na            & \na                                            \\
			\midrule
			ReSample~\citep{song2024solving}
			 & 24.49                                            & 0.51              & 22.26          & 0.48          & 22.43          & 0.44          & 22.68          & 0.58          & 20.53          & 0.38          \\
			Tunable ReSample (Ours)
			 & \textbf{28.43}                                   & \textbf{0.28}     & \textbf{24.84} & \textbf{0.39} & \textbf{24.53} & \textbf{0.39} & \textbf{28.51} & \textbf{0.29} & \textbf{20.70} & \textbf{0.36} \\
			\bottomrule
		\end{tabular}%
	}
	\captionsetup{width=1.0\textwidth}
	\caption{\textbf{Quantitative comparison on FFHQ across five inverse problems.}
		Mean PSNR~$\uparrow$ and LPIPS~$\downarrow$ are reported over 100 test images.
		Tunable priors are compared against their fixed-complexity counterparts using the same
		inversion method; bold indicates the better result within each matched pair, and dashes
		denote settings not evaluated.}
	\label{tab:ffhq}
\end{table}

\begin{table}[t!]
	\centering
	\scriptsize
	\setlength{\tabcolsep}{1.5pt}
	\renewcommand{\arraystretch}{1}
	\newcommand{\na}{\multicolumn{2}{c}{---}}
	\resizebox{1.0\textwidth}{!}{%
		\begin{tabular}{@{}l*{10}{c}@{}}
			\toprule
			\multirow{2}{*}{\textbf{Method}}
			 & \multicolumn{2}{c}{\textbf{Compressed Sensing }}
			 & \multicolumn{2}{c}{\textbf{Inpainting}}
			 & \multicolumn{2}{c}{\textbf{Super Resolution }}
			 & \multicolumn{2}{c}{\textbf{Phase Retrieval}}
			 & \multicolumn{2}{c}{\textbf{HDR}}                                                                                                                                                                         \\
			\cmidrule(lr){2-3}
			\cmidrule(lr){4-5}
			\cmidrule(lr){6-7}
			\cmidrule(lr){8-9}
			\cmidrule(lr){10-11}
			 & PSNR$\uparrow$                                   & LPIPS$\downarrow$
			 & PSNR$\uparrow$                                   & LPIPS$\downarrow$
			 & PSNR$\uparrow$                                   & LPIPS$\downarrow$
			 & PSNR$\uparrow$                                   & LPIPS$\downarrow$
			 & PSNR$\uparrow$                                   & LPIPS$\downarrow$                                                                                                                                     \\
			\midrule
			LatentDAPS~\citep{zhang2025improving}
			 & 28.01                                            & \textbf{0.24}     & 22.65          & 0.48          & 21.15          & 0.54          & 23.67          & 0.43          & 18.84          & 0.32          \\
			Tunable LatentDAPS (Ours)
			 & \textbf{28.12}                                   & \textbf{0.24}     & \textbf{22.99} & \textbf{0.42} & \textbf{21.17} & \textbf{0.50} & \textbf{23.96} & \textbf{0.38} & \textbf{18.96} & \textbf{0.31} \\
			\midrule
			PSLD~\citep{rout2023solving}
			 & 18.61                                            & 0.61              & 18.82          & 0.63          & 18.56          & 0.64          & \na            & \na                                            \\
			Tunable PSLD (Ours)
			 & \textbf{19.39}                                   & \textbf{0.60}     & \textbf{19.25} & \textbf{0.61} & \textbf{19.34} & \textbf{0.61} & \na            & \na                                            \\
			\midrule
			ReSample~\citep{song2024solving}
			 & 23.74                                            & 0.47              & 19.55          & 0.49          & 18.82          & 0.55          & 23.00          & 0.49          & 20.62          & 0.30          \\
			Tunable ReSample (Ours)
			 & \textbf{25.86}                                   & \textbf{0.36}     & \textbf{22.19} & \textbf{0.43} & \textbf{21.62} & \textbf{0.47} & \textbf{26.49} & \textbf{0.33} & \textbf{20.71} & \textbf{0.29} \\
			\bottomrule
		\end{tabular}%
	}
	\captionsetup{width=1.0\textwidth}
	\caption{\textbf{Quantitative comparison on ImageNet across five inverse problems.}
		Mean PSNR~$\uparrow$ and LPIPS~$\downarrow$ are reported over 100 test images.
		Tunable variants are evaluated against the corresponding fixed-complexity baselines,
		isolating the effect of latent-space tunability; bold indicates the better result within
		each matched pair, and dashes denote settings not evaluated.}
	\label{tab:img_net}
\end{table}
\FloatBarrier

\section{Conclusion and Limitations}
\label{sec:conclusion}

Prior work on inverse problems with generative priors has primarily focused on fixed-complexity models, with most effort directed at inversion algorithms. In this work, we study latent tunable priors as a complementary axis of improvement: rather than changing the inversion procedure, we vary the effective latent dimensionality of the prior, giving practitioners an additional degree of freedom with minimal training overhead and no increase in the number of model parameters. Using nested dropout as an existing mechanism for learning ordered latent representations, we demonstrate that tunability can improve reconstruction error across several model classes and inverse problems, with intermediate latent dimensions often outperforming both low- and high-complexity extremes. Our goal is not to introduce nested dropout as a new method, but to use it to test the hypothesis that controlling latent dimensionality can meaningfully improve inverse-problem performance. The gains are largest for inversion methods that solve a measurement-consistency minimization problem at each step: ReSample~\citep{song2024solving}, flow-based MAP~\citep{NEURIPS2024_698570ae}, and \cref{algo:posterior}. LatentDAPS~\citep{zhang2025improving} is less consistent, which we observed empirically.

Several limitations point to important directions for future work. Our theory currently covers only linear generators under additive Gaussian noise; its value is mainly in suggesting that analogous criteria for selecting the optimal latent dimension $k$ may be possible for a richer model class. We focus on latent priors, where complexity is naturally controlled by the number of retained coordinates, while extending tunability to pixel-space architectures would require a different mechanism. We also treat the latent dimensionality $k$ as a solver hyperparameter, selecting a single value for each dataset, inverse problem, and noise level using a small held-out validation set with ground truth. While this avoids any use of ground truth at test time, it assumes representative validation data; developing measurement-only or instance-adaptive selection rules for $k$ remains an important open problem. Finally, our experiments are limited to resolutions up to $256 \times 256$, leaving open how latent tunability scales to higher-resolution data and more complex reconstruction tasks.

\section*{Acknowledgments}
SG and PH acknowledge support from NSF Awards DMS-1848087 and DMS-2022205.

\bibliography{ref.bib}
\bibliographystyle{unsrtnat}
\newpage
\appendix
\section{Appendix}
All FID scores reported in this appendix are computed with the \texttt{clean-fid} toolkit~\citep{Parmar_2022_CVPR}.
\subsection{VAE Training/Inversion}
Our VAE follows the \texttt{AutoencoderKL} architecture from~\citet{von-platen-etal-2022-diffusers}. Both VAEs are trained as described in the main text with the regularization weight $\lambda$ fixed at $0.1$ and a nested-dropout success probability of $10^{-3}$, and use a VGG-based perceptual loss to encourage perceptual fidelity. For inversion, we use gradient descent on the latent vector, following the reference implementation accompanying~\citet{bora2017compressed}.

\textbf{CelebA.} A latent dimensionality of $16\times 16\times 16$ is used, yielding a model with approximately $44$M parameters.

\textbf{MS COCO.} A latent dimensionality of $24\times 16\times 16$ is used.

\begin{figure}[h!]
	\centering
	\begin{subfigure}[b]{0.48\textwidth}
		\centering
		\includegraphics[width=\textwidth]{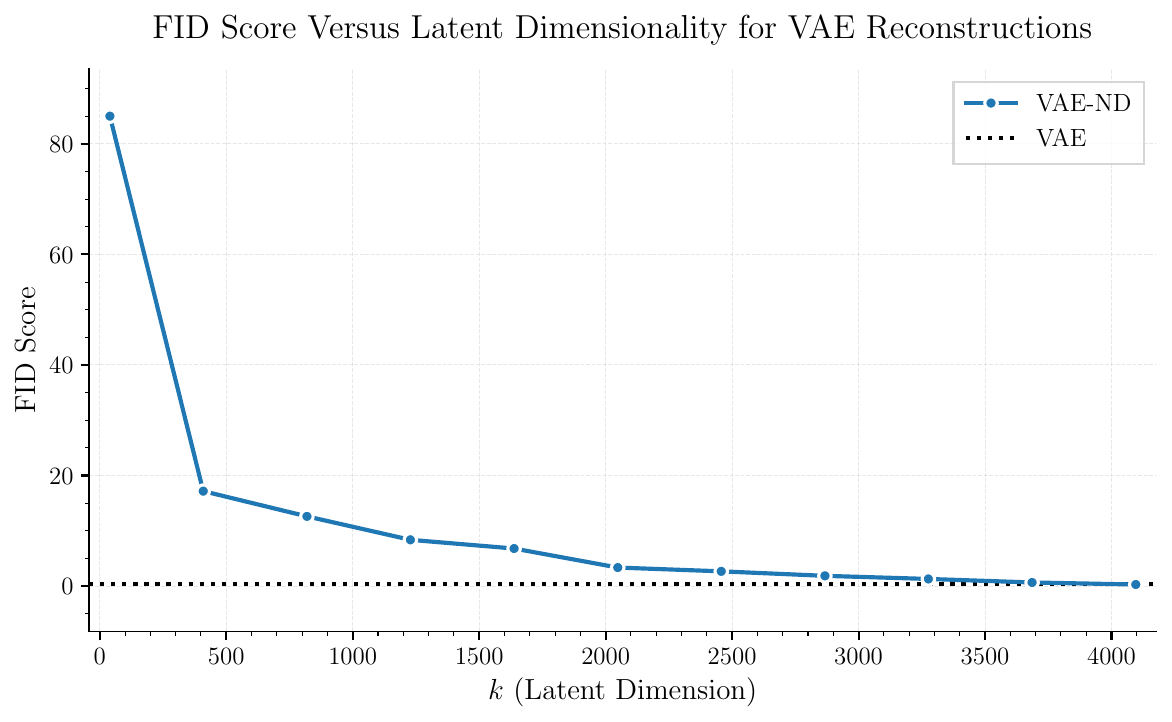}
		\caption{CelebA dataset}
		\label{fig:ldm_train_k_celeba}
	\end{subfigure}
	\hfill
	\begin{subfigure}[b]{0.48\textwidth}
		\centering
		\includegraphics[width=\textwidth]{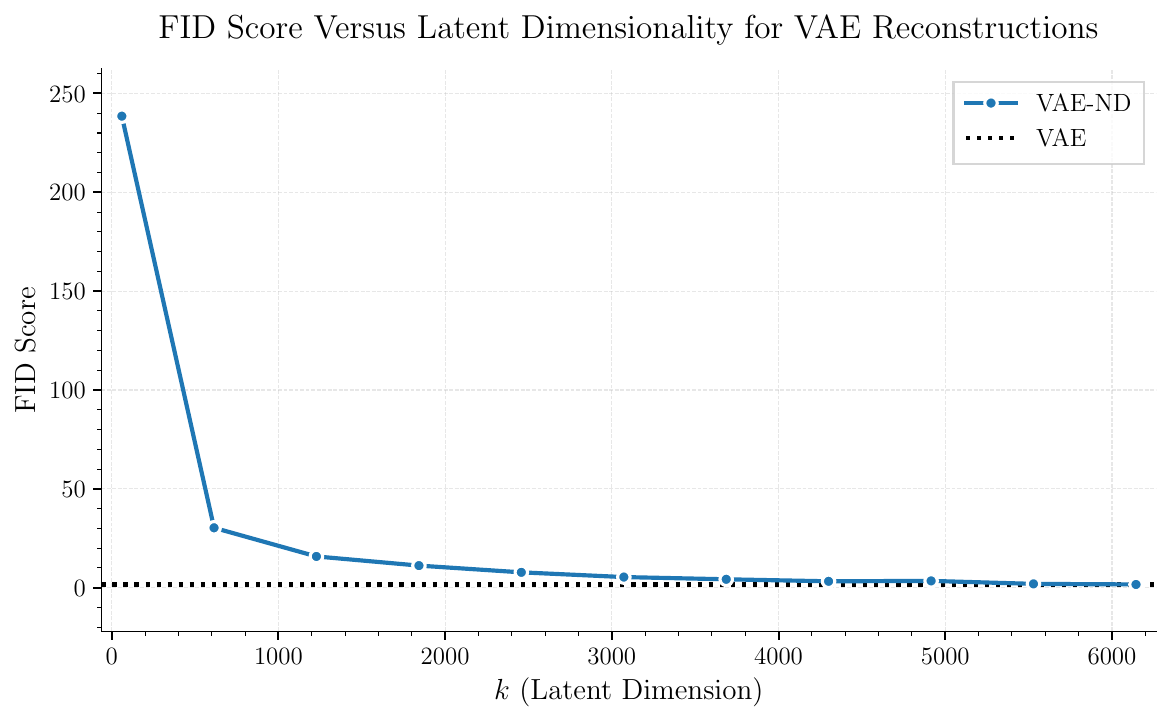}
		\caption{MS COCO dataset}
		\label{fig:ldm_train_k_coco}
	\end{subfigure}
	\caption{Reconstruction FID as a function of latent dimensionality $k$.}
\end{figure}

\begin{figure}[h!]
	\centering
	\includegraphics[width=1.0\textwidth,height=0.3\linewidth,keepaspectratio]{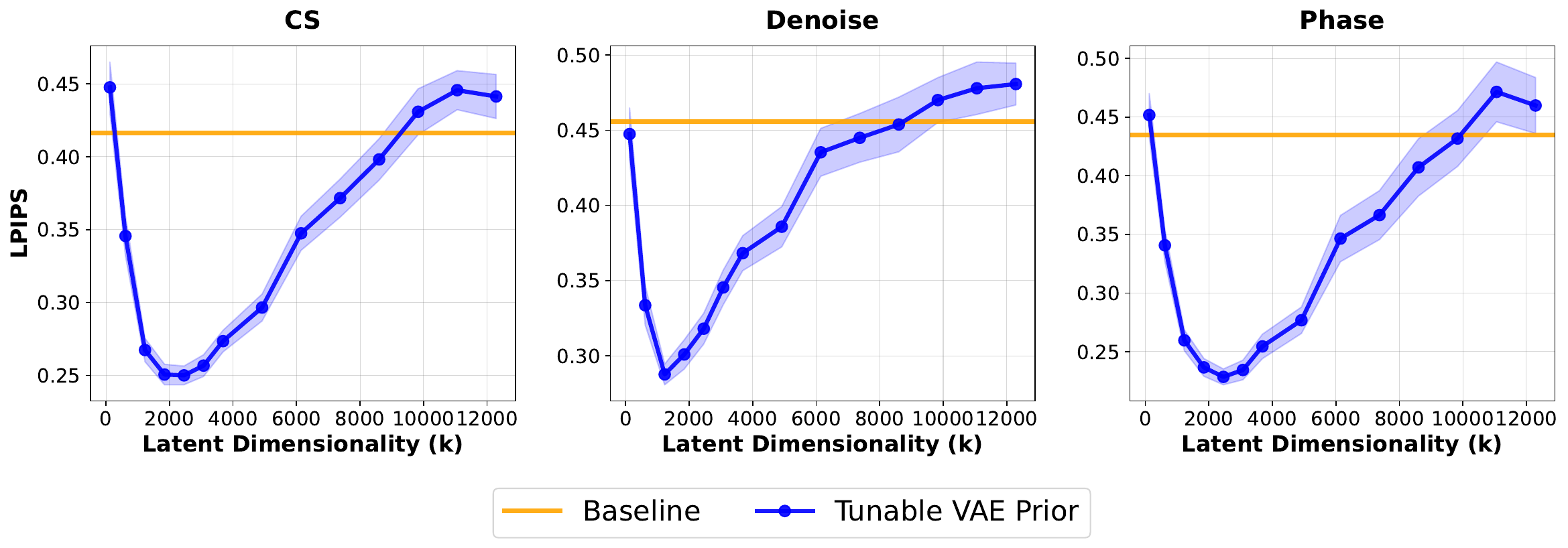}
	\caption{
		Performance of a generative prior with tunable complexity (Tunable VAE Prior) compared to its fixed-complexity counterpart (VAE) on compressed sensing and denoising tasks on the CelebA dataset. The tunable prior achieves improved reconstructions over a range of latent dimensionalities $k$, as measured by LPIPS, relative to the fixed-complexity baseline.
	}
	\label{fig:vae_flow_ushape}
\end{figure}

\vspace{2em}
\subsection{LDM Training/Inversion}
\label{app:LDM}

\begin{algorithm}[t!]
	\caption{Tunable Posterior Sampling for Latent Diffusion (Simplified Template)}
	\label{algo:posterior}
	\begin{algorithmic}[1]
		\State \textbf{Input:} $\vy$, $\gA$, $\gD$, $\beps_\theta$, steps $T$, schedule $\{\bar\alpha_t\}_{t=0}^T$, parameter $k$
		\Statex \hspace{\algorithmicindent} Reverse sampler (DDIM or DDPM), DDIM $\eta\in[0,1]$, correction scales $\tau_t>0$
		\State \textbf{Output:} $\gD(\vz_0)$
		\State $\vz_T\sim\mathcal{N}(0,\mI_d)$; $\vz_T\gets(\vz_T)_{\downarrow k}$
		\For{$t=T,\ldots,1$}
		\State $\hat{\beps}_t\gets\beps_\theta(\vz_t,t)$
		\State $\hat{\vz}_0\gets(\vz_t-\sqrt{1-\bar\alpha_t}\,\hat{\beps}_t)/\sqrt{\bar\alpha_t}$
		\State $\beps'_t\sim\mathcal{N}(0,\mI_d)$
		\If{reverse sampler is DDIM}
		\State $\sigma_t\gets\eta\sqrt{\frac{1-\bar\alpha_{t-1}}{1-\bar\alpha_t}(1-\alpha_t)}$
		\State $\vz'_{t-1}\gets\sqrt{\bar\alpha_{t-1}}\hat{\vz}_0+\sqrt{1-\bar\alpha_{t-1}-\sigma_t^2}\hat{\beps}_t+\sigma_t\beps'_t$
		\Else \Comment{DDPM}
		\State $\vz'_{t-1}\gets\frac{\sqrt{\alpha_t}(1-\bar\alpha_{t-1})}{1-\bar\alpha_t}\vz_t
		+\frac{\sqrt{\bar\alpha_{t-1}}\beta_t}{1-\bar\alpha_t}\hat{\vz}_0+\sigma_t^{\text{DDPM}}\beps'_t$
		\EndIf
		\State Initialize $\vz$ at $\vz'_{t-1}$ and approximately minimize over $\vz\in\R^d$:
		\Statex \hspace{\algorithmicindent} $\displaystyle
		\|\vy-\gA(\gD(\vz))\|_2^2+\frac{\|\vz-\vz'_{t-1}\|_2^2}{2\tau_t^2}$
		\State $\vz_{t-1}\gets$ the resulting approximate minimizer
		\State $\vz_{t-1}\gets(\vz_{t-1})_{\downarrow k}$
		\EndFor
		\State \textbf{return} $\gD(\vz_0)$
	\end{algorithmic}
\end{algorithm}

Here $(\cdot)_{\downarrow k}$ retains the first $k$ latent coordinates and zeros the rest, $\gD$ is the autoencoder's decoder, and $\beps_\theta$ is the trained noise-prediction network. We use $\bar\alpha_0=1$ and the previous-selected-timestep convention described in \cref{subsec:IP_TP}; $\sigma_t^{\text{DDPM}}$ is defined in \cref{sub:DM}. The positive correction scales $\tau_t$ are supplied separately from the sampling standard deviations, so the objective remains finite when the sampling variance is zero, including at the terminal step. \Cref{algo:posterior} is a simplified mathematical template with approximate optimization initialized from the reverse proposal.

Following~\citet{von-platen-etal-2022-diffusers}, all small-scale diffusion priors used in this section share a common training recipe: a \texttt{UNet2DModel} with approximately $200$M parameters, optimized with AdamW at learning rate $10^{-4}$ and batch size $128$ under the \texttt{DDPMScheduler}. We compute FID with the \texttt{DPMSolverMultistepScheduler}~\citep{lu2022dpm} at $100$ sampling steps, evaluating every $20{,}000$ iterations and retaining the three checkpoints with the lowest validation FID; the best-FID checkpoint is used for all subsequent inverse-problem evaluations.

\textbf{CelebA.} Training-set FID is computed over $50{,}000$ images and test-set FID over $10{,}000$ unseen images from the PyTorch CelebA split.

\textbf{CIFAR-10 and MS COCO.} The same training recipe and FID protocol are used, applied to each dataset's standard train/test partitions.

\begin{figure}[h!]
	\centering
	\begin{subfigure}[b]{0.48\textwidth}
		\centering
		\includegraphics[width=\textwidth]{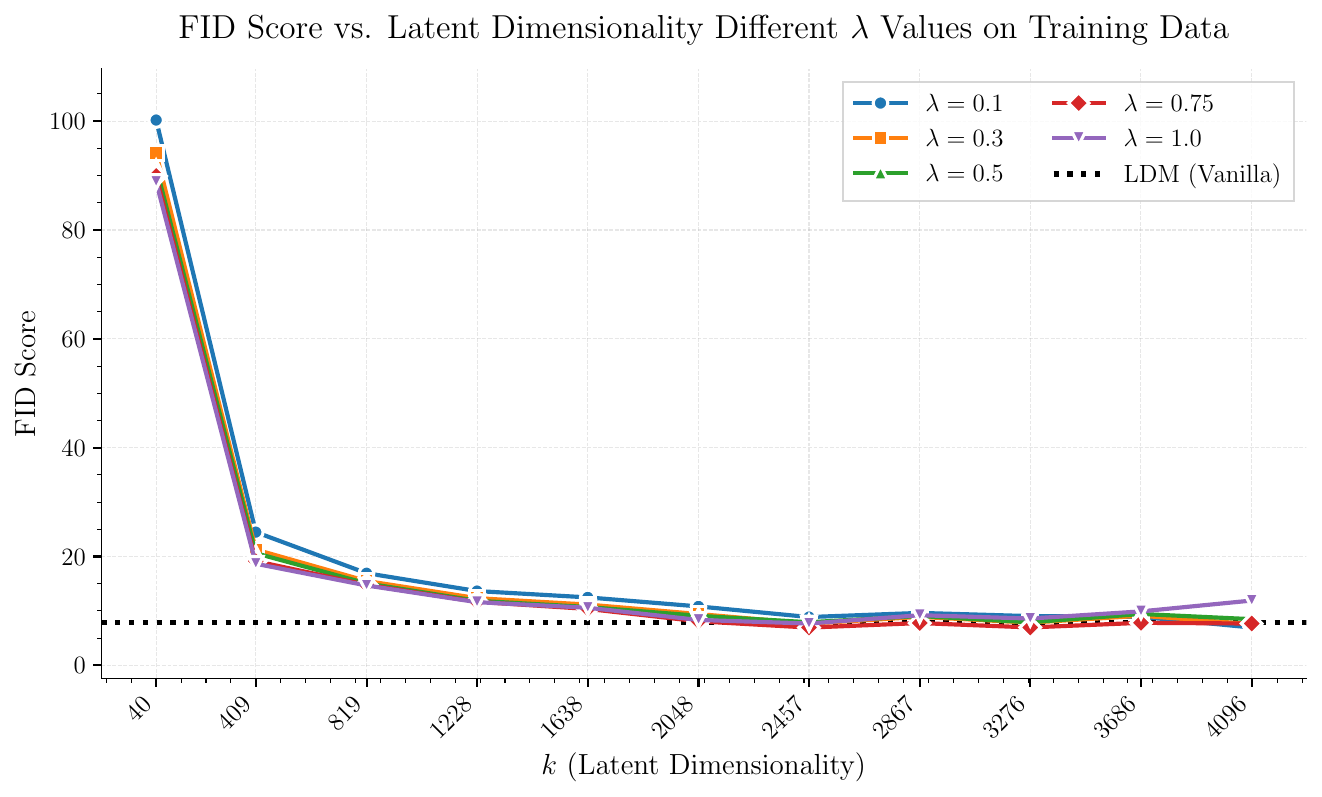}
		\caption{Training set CelebA}
		\label{fig:ldm_train_k}
	\end{subfigure}
	\hfill
	\begin{subfigure}[b]{0.48\textwidth}
		\centering
		\includegraphics[width=\textwidth]{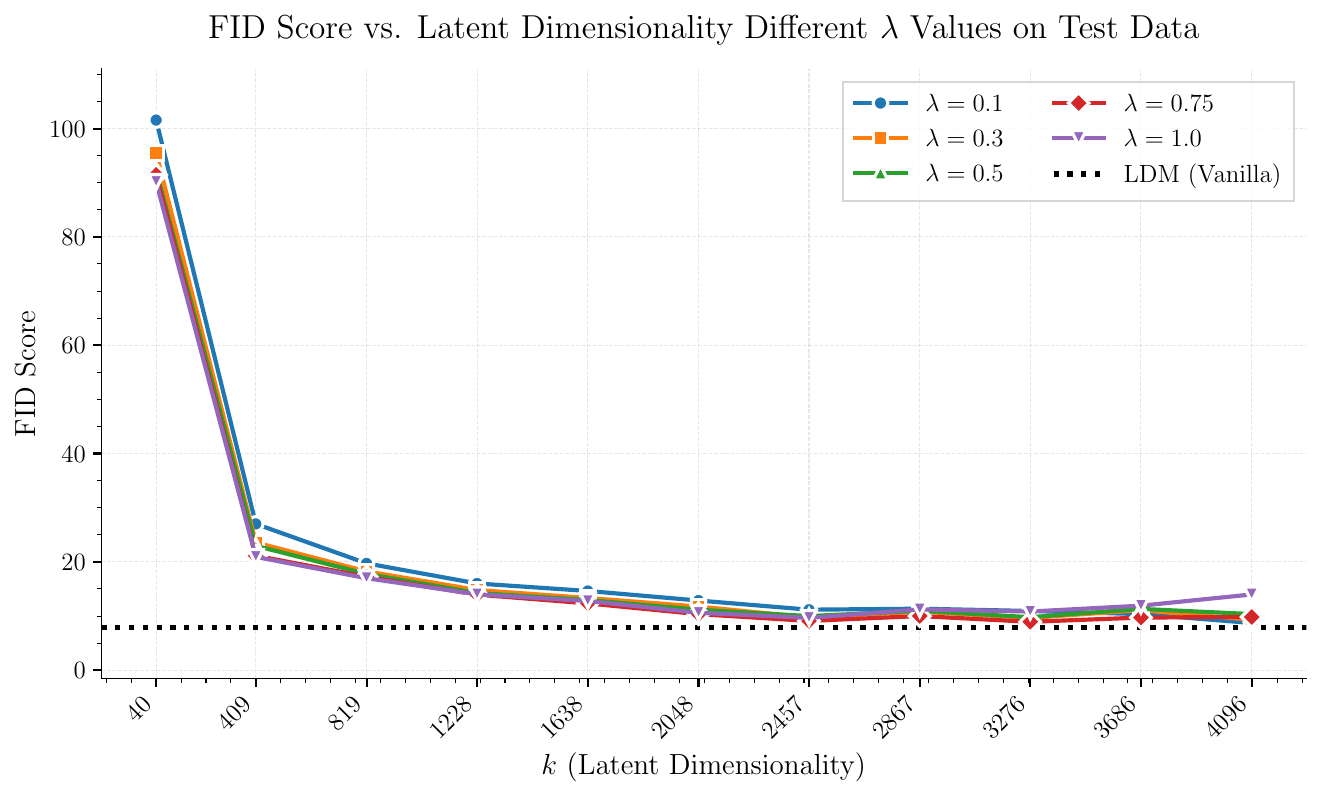}
		\caption{Test set CelebA}
		\label{fig:ldm_test_k}
	\end{subfigure}
	\caption{FID score as a function of latent dimensionality $k$.}
	\label{fig:ldm_k_comparison}
\end{figure}

For the inversion algorithms DPS~\citep{chung2023diffusion} and LDPS/PSLD~\citep{rout2023solving}, we used the official codebases but adapted them to the Diffusers framework, which already provides existing pipelines. All hyperparameters were selected by grid search over the configuration files provided on GitHub, and we report results using the best-performing settings. Our \href{https://github.com/g33sean/Latent-Generative-Models-with-Tunable-Complexity-for-Compressed-Sensing-and-other-Inverse-Problems}{public code repository} provides training, evaluation, inverse-problem implementations, configurations, and environment instructions. In our experiments, DDIM sampling consistently yielded stronger reconstruction quality at intermediate latent dimensionalities compared to stochastic samplers, though a full theoretical understanding of this effect remains an open question. For \cref{algo:posterior}, the best-performing setup used the Adam optimizer~\citep{KingBa15} with the DDIM reverse equation, 3 optimization steps per update, 500 reverse steps, and DDIM stochasticity parameter $\eta=0.1$. In these experiments, a small amount of DDIM stochasticity improved reconstruction: setting $\eta=0$, \ie, deterministic DDIM sampling, led to unstable or degraded reconstructions. This approach is particularly effective in our setting due to the relatively large latent space compared to the image space. The overall objective was to implement a straightforward inversion algorithm that can be readily applied across a wide range of problems.

\paragraph{FFHQ and ImageNet Model Configurations}
For the FFHQ and ImageNet experiments we train two diffusion U-Nets that operate in a shared $24\times 64\times 64$ latent space produced by a KL autoencoder ($256\times 256$ RGB images, $4\times$ spatial downsampling). The autoencoder is trained for $100{,}000$ steps at batch size $32$ with a VGG-based perceptual loss, a small KL weight of $10^{-6}$, an adversarial discriminator (weight $0.1$, ramped in over the first $300$ steps), and a nested-dropout regularizer with strength $\lambda=0.1$. Both diffusion models are trained for $1000$ DDPM timesteps with min-SNR-$\gamma$ weighting ($\gamma=5$), AdamW at learning rate $10^{-4}$, gradient accumulation of $2$, and a linear warmup followed by cosine decay; we sample with the \texttt{DPMSolverMultistepScheduler} for FID evaluation and with DDIM for inverse problems.

\textbf{FFHQ.} The FFHQ U-Net has approximately $520$M parameters and is used for \emph{both} the tunable and fixed-complexity (baseline) prior; the two variants share an identical architecture and differ only in whether nested-dropout is applied to the latent during training. The network has four resolution stages ($64{\rightarrow}32{\rightarrow}16{\rightarrow}8$) with channel widths $(320, 640, 960, 1280)$, two residual blocks per stage, attention at every stage with head dimension $64$, and a scale--shift time embedding. We train for $80{,}000$ optimizer steps at a per-GPU batch size of $100$ (effective batch $200$ with gradient accumulation) with $10{,}000$ warmup steps, achieving an FID of $8$ on a held-out validation set of $5{,}000$ images.

\textbf{ImageNet.} The ImageNet U-Net has approximately $720$M parameters and is class-conditional over the $1000$ ImageNet classes plus a learned null embedding to support classifier-free guidance~\citep{ho2022classifier}. At training time we replace the class label with the null token with probability $0.1$; at sampling time we use a guidance scale of $4.0$. The network has six resolution stages ($64{\rightarrow}32{\rightarrow}16{\rightarrow}8{\rightarrow}4{\rightarrow}2$) with channel widths $(320, 320, 640, 640, 1280, 1280)$; the top two stages use plain residual blocks while the bottom four add attention with head dimension $64$, and each stage has two residual blocks. We train for $120{,}000$ optimizer steps at a per-GPU batch size of $160$ (effective batch $320$ with accumulation) with $5{,}000$ warmup steps, achieving an FID of $18$ on a held-out validation set of $5{,}000$ images. Both models are trained from scratch on a single NVIDIA H200, and all inverse-problem inference is run on a single NVIDIA A100 ($40$\,GB).

\paragraph{Tunable Priors on Various Datasets}
Each problem is solved on $5$ images randomly drawn from the held-out test set with a batch size of $12$. The blue shaded region represents one standard deviation.

\begin{center}
	\begin{minipage}{\textwidth}
		\centering
		\includegraphics[width=0.7\linewidth]{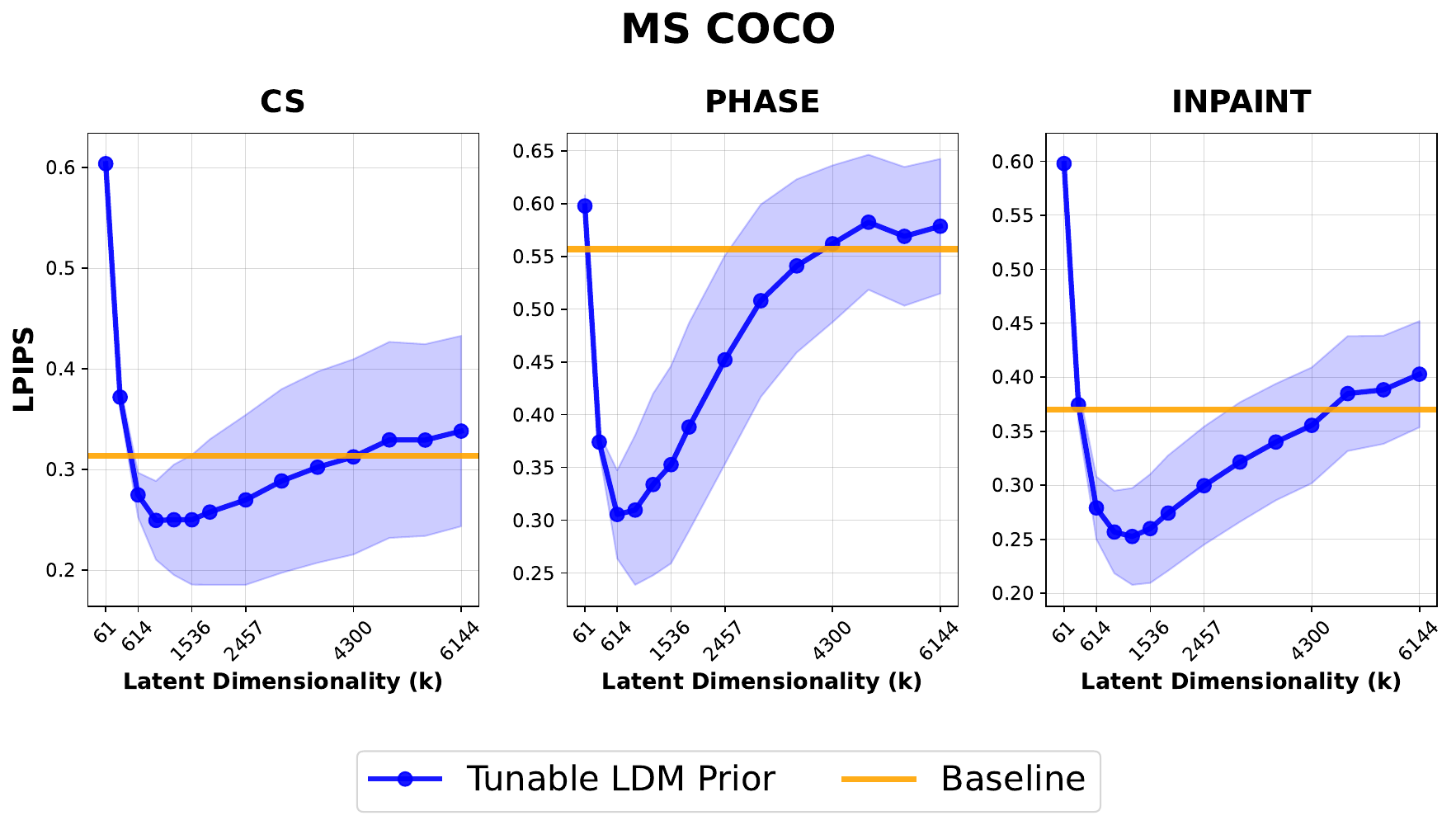}
		\captionof{figure}{Performance of a generative prior with tunable complexity (Tunable LDM Prior) and its fixed-complexity counterpart (Baseline) for compressed sensing and inpainting on MS COCO. The tunable prior demonstrates a range of parameters $k$ that yield better LPIPS scores than the baseline.}
		\label{fig:ms_coco_ushape}
	\end{minipage}
\end{center}
\vspace{1em}
\begin{center}
	\begin{minipage}{\textwidth}
		\centering
		\includegraphics[width=0.7\linewidth]{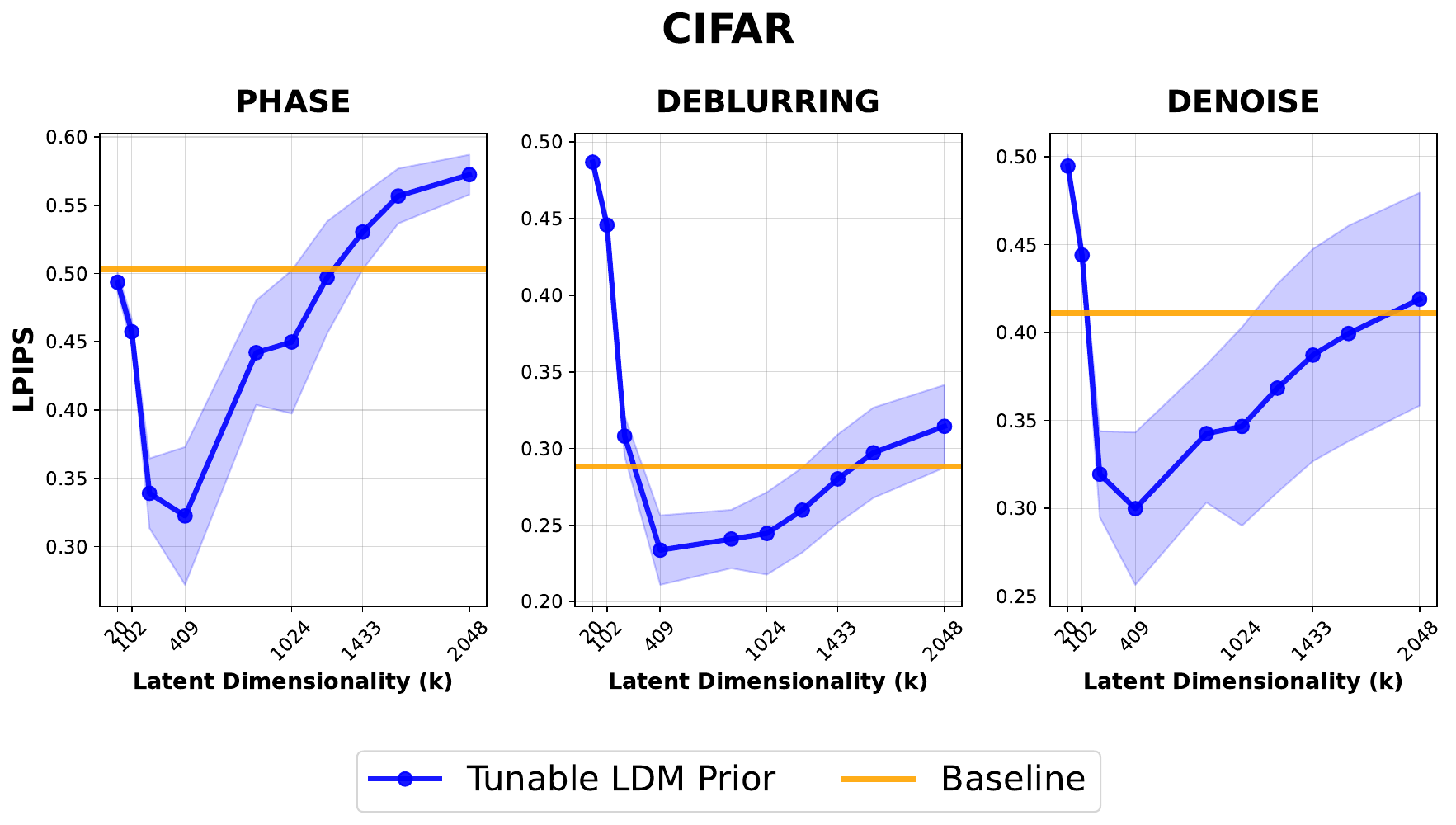}
		\captionof{figure}{Performance of a generative prior with tunable complexity (Tunable LDM Prior) and its fixed-complexity counterpart (Baseline) for compressed sensing and inpainting on CIFAR-10. The tunable prior demonstrates a range of parameters $k$ that yield better LPIPS scores than the baseline.}
		\label{fig:cifar_ushape}
	\end{minipage}
\end{center}

\paragraph{How does the number of measurements affect where the optimal latent dimensionality for reconstruction is?}

\begin{center}
	\begin{minipage}{\textwidth}
		\centering
		\includegraphics[width=0.85\textwidth]{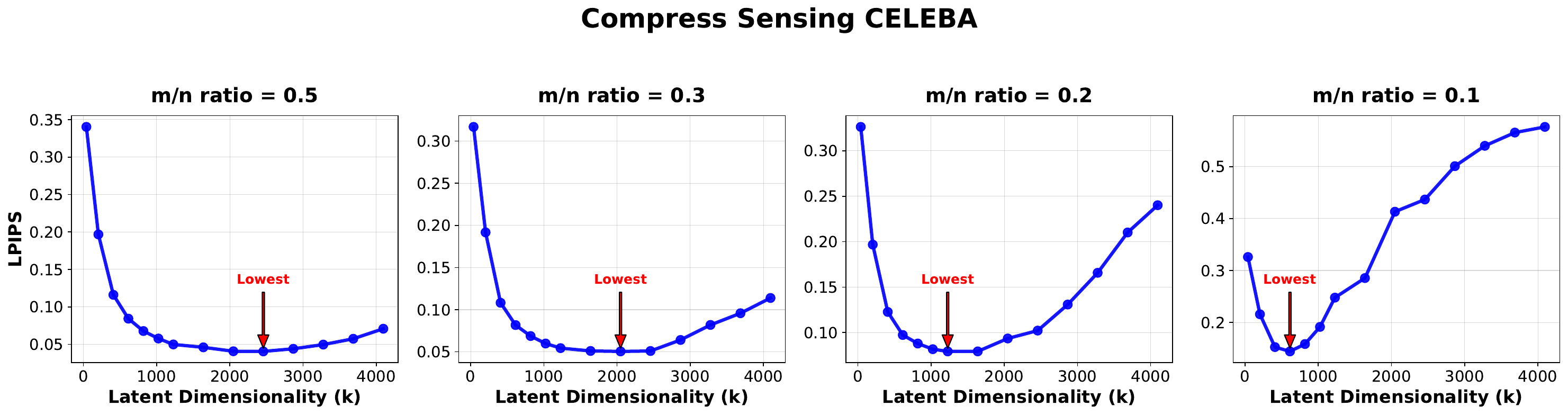}
		\captionof{figure}{Compressed sensing performance as a function of latent dimensionality $k$ for different measurement regimes.}
		\label{fig:cs_rebuttal}
	\end{minipage}
\end{center}
\vspace{1em}
\begin{center}
	\begin{minipage}{\textwidth}
		\centering
		\includegraphics[width=0.85\textwidth]{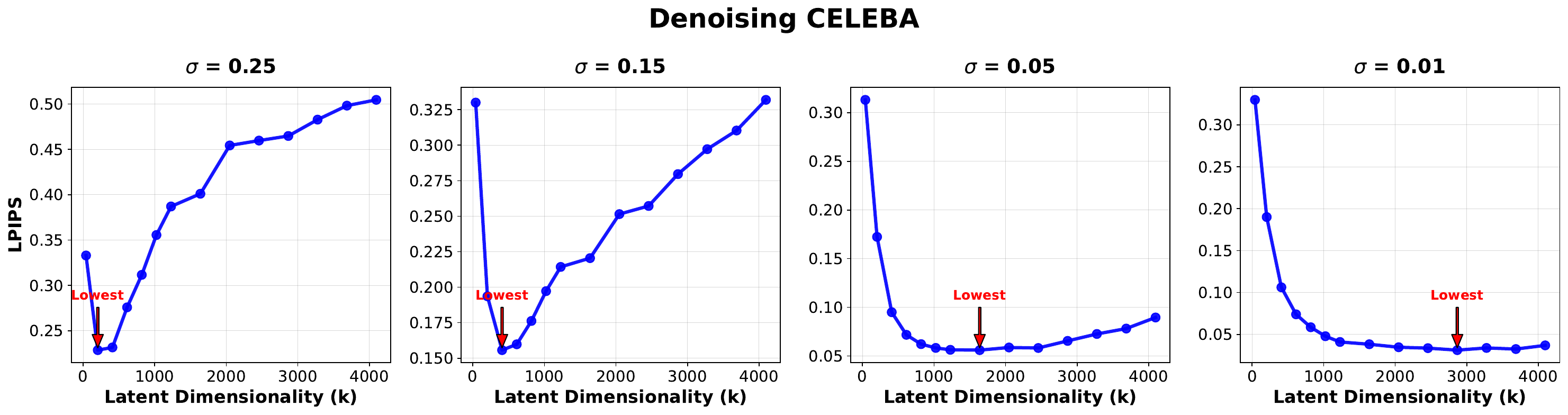}
		\captionof{figure}{Denoising performance as a function of latent dimensionality $k$ for different measurement regimes.}
		\label{fig:denoise_rebuttal}
	\end{minipage}
\end{center}

\paragraph{How the optimal $k$ was selected for \cref{tab:ffhq,tab:img_net}.}
We select the latent dimensionality $k$ using a held-out validation set.
Specifically, the value of $k$ achieving the lowest mean LPIPS score on the validation set is used for the results reported in \cref{tab:ffhq,tab:img_net}.

\begin{figure}[h!]
	\centering
	\includegraphics[width=0.7\linewidth]{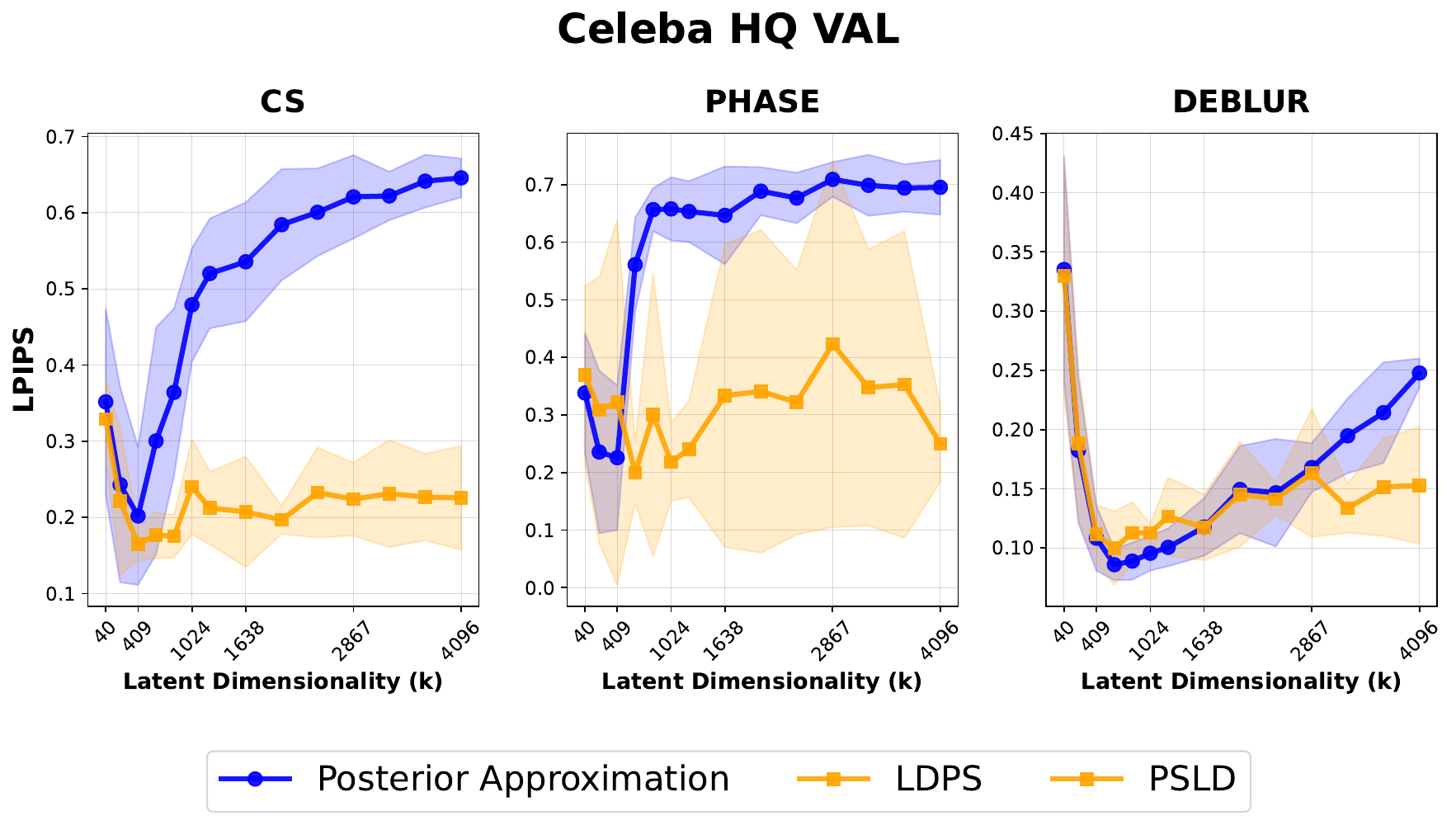}
	\caption{Validation performance used to select the latent dimensionality $k$.}
	\label{fig:val_k_selection}
\end{figure}

\begin{figure}[h!]
	\centering
	\includegraphics[width=0.7\linewidth]{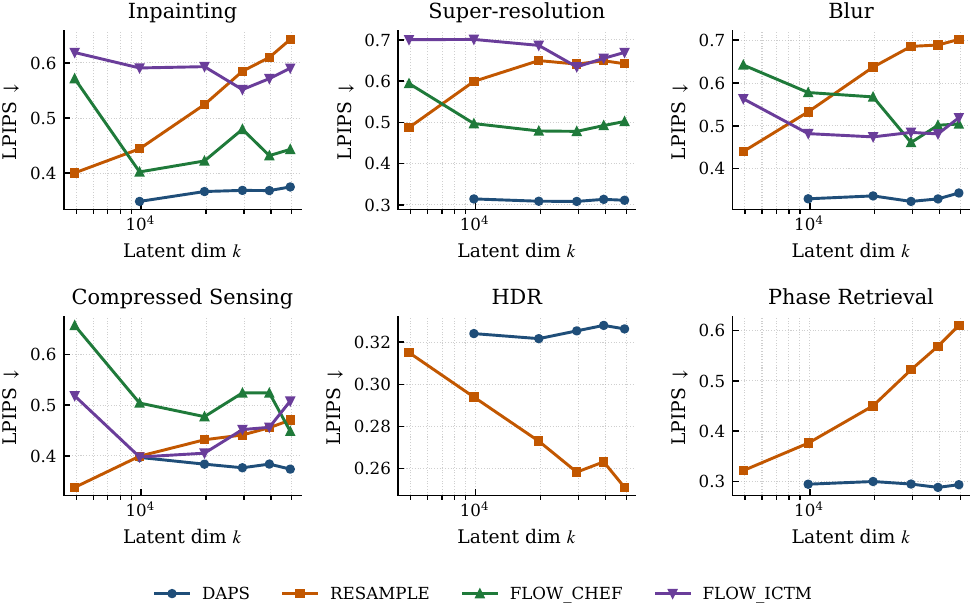}
	\caption{Simulation results corroborating the validation-based $k$-selection procedure used for \cref{tab:ffhq,tab:img_net}.}
	\label{fig:neurips_sim}
\end{figure}

\begin{figure}[p] 
	\centering

	\begin{minipage}{\textwidth}
		\centering
		\includegraphics[width=0.85\linewidth,height=0.65\linewidth,keepaspectratio]{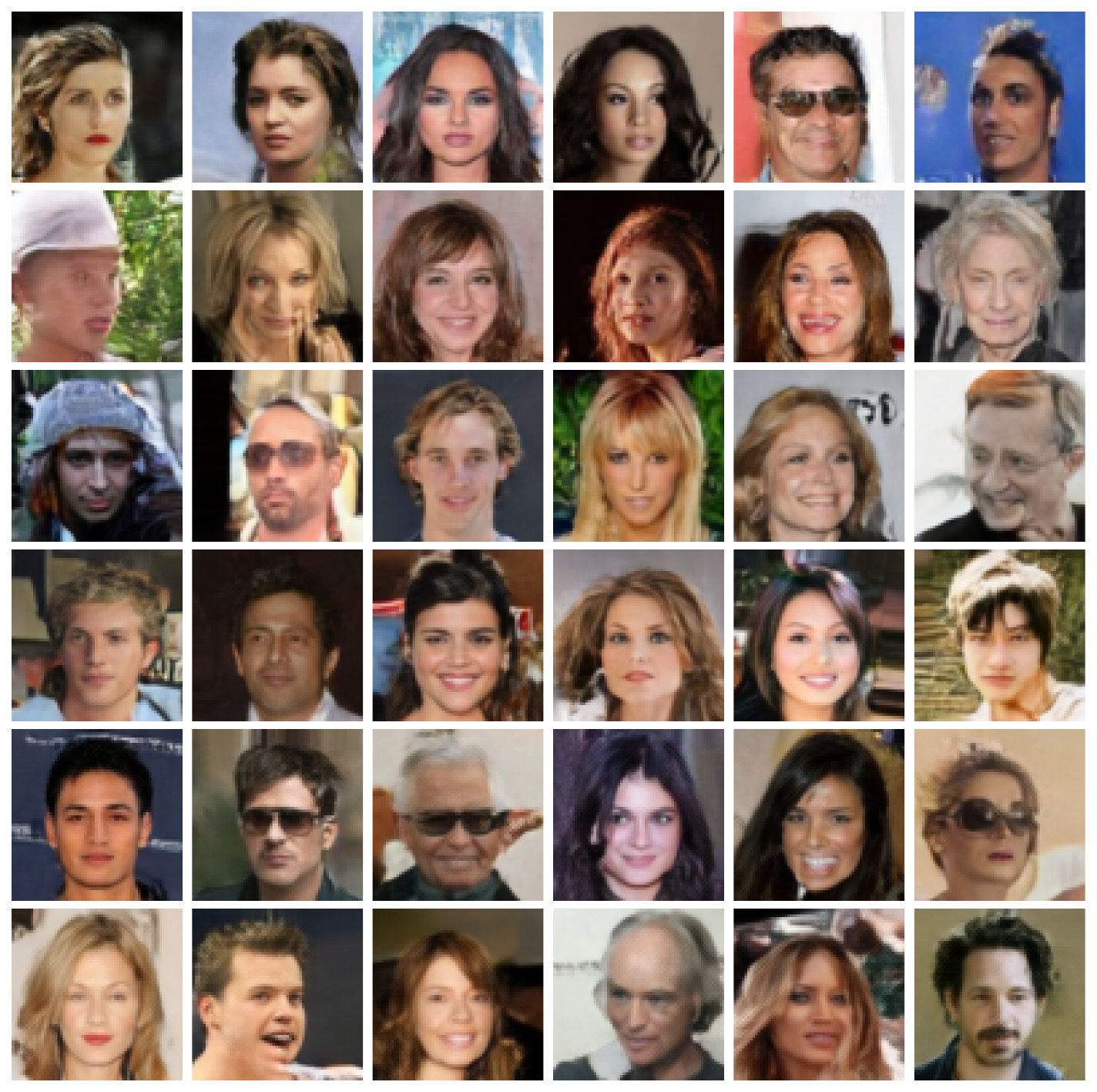}
		\caption{Generated samples from the latent diffusion model with different latent dimensionalities.
			Each row shows images produced when only the first $k$ of the $4096$ latent dimensions are kept during sampling
			($k=4096$: 100\%, $k=3600$: 88\%, $k=3000$: 73\%, $k=2500$: 61\%, $k=2000$: 49\%, $k=1500$: 37\%.)}
	\end{minipage}

	\vspace{2em} 

	\begin{minipage}{\textwidth}
		\centering
		\includegraphics[width=0.85\linewidth,height=0.5\linewidth,keepaspectratio]{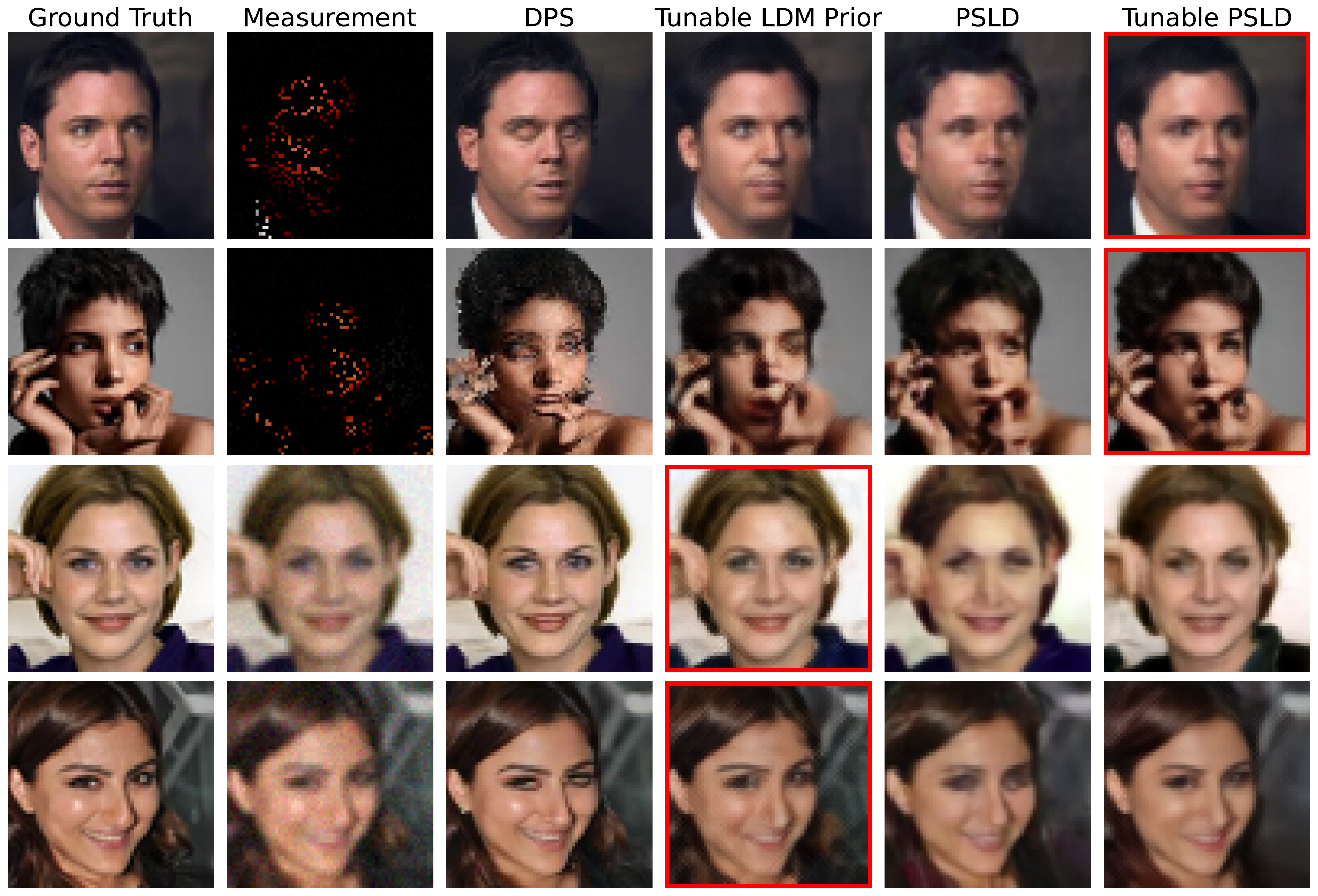}
		\caption{Qualitative results based on Table 2. Red indicates the lowest LPIPS score.}
		\label{fig:tunable_CELBA_vis}
	\end{minipage}
\end{figure}

\clearpage
\begin{table*}[htbp]
	\centering
	\small
	\setlength{\tabcolsep}{5pt}
	\renewcommand{\arraystretch}{1.3}
	\resizebox{\textwidth}{!}{
		\begin{tabular}{@{}lcccccccc@{}}
			\toprule
			 & \multicolumn{2}{c}{CS}
			 & \multicolumn{2}{c}{Phase Retrieval}
			 & \multicolumn{2}{c}{Inpainting (random)}
			 & \multicolumn{2}{c}{Deblur (Gauss.)}                                                             \\
			\cmidrule(lr){2-3}
			\cmidrule(lr){4-5}
			\cmidrule(lr){6-7}
			\cmidrule(lr){8-9}
			Method
			 & PSNR$\uparrow$                                 & LPIPS$\downarrow$
			 & PSNR$\uparrow$                                 & LPIPS$\downarrow$
			 & PSNR$\uparrow$                                 & LPIPS$\downarrow$
			 & PSNR$\uparrow$                                 & LPIPS$\downarrow$                              \\
			\midrule
			DPS~\citep{chung2023diffusion}
			 & $25.65 \pm 2.67$                               & \colorbox{blue!20}{$\mathbf{0.159 \pm 0.067}$}
			 & $22.15 \pm 3.84$                               & $0.267 \pm 0.061$
			 & $23.18 \pm 2.11$                               & $0.146 \pm 0.042$
			 & $27.90 \pm 1.53$                               & $0.093 \pm 0.019$                              \\

			Tunable LDM Prior (Ours)
			 & $25.49 \pm 1.70$                               & $0.179 \pm 0.045$
			 & $25.24 \pm 1.71$                               & \colorbox{blue!20}{$\mathbf{0.168 \pm 0.046}$}
			 & $25.30 \pm 1.76$                               & $0.122 \pm 0.028$
			 & \colorbox{blue!20}{$\mathbf{29.52 \pm 1.11}$}
			 & \colorbox{blue!20}{$\mathbf{0.085 \pm 0.024}$}                                                  \\

			LDPS / PSLD~\citep{rout2023solving}
			 & $23.30 \pm 1.24$                               & $0.248 \pm 0.063$
			 & $24.48 \pm 1.35$                               & $0.169 \pm 0.038$
			 & $25.85 \pm 2.18$                               & $0.136 \pm 0.036$
			 & $27.33 \pm 1.23$                               & $0.165 \pm 0.037$                              \\

			Tunable LDPS / PSLD (Ours)
			 & $24.69 \pm 1.37$                               & $0.164 \pm 0.036$
			 & $21.95 \pm 1.31$                               & $0.256 \pm 0.065$
			 & \colorbox{blue!20}{$\mathbf{26.75 \pm 1.71}$}
			 & \colorbox{blue!20}{$\mathbf{0.095 \pm 0.028}$}
			 & $28.50 \pm 1.05$                               & $0.115 \pm 0.031$                              \\

			NF~\citep{asim2020invertible}
			 & $21.31 \pm 1.94$                               & $0.508 \pm 0.079$
			 & $20.67 \pm 2.32$                               & $0.554 \pm 0.073$
			 & $20.07 \pm 2.42$                               & $0.376 \pm 0.091$
			 & $21.43 \pm 2.11$                               & $0.438 \pm 0.034$                              \\

			Tunable NF (Ours)
			 & \colorbox{blue!20}{$\mathbf{27.16 \pm 1.56}$}  & $0.238 \pm 0.058$
			 & \colorbox{blue!20}{$\mathbf{26.11 \pm 2.07}$}  & $0.246 \pm 0.076$
			 & $23.24 \pm 2.43$                               & $0.167 \pm 0.057$
			 & $24.85 \pm 2.21$                               & $0.215 \pm 0.043$                              \\
			\bottomrule
		\end{tabular}
	}
	\caption{Quantitative results on CelebA-HQ across inverse problems. Values reported as mean $\pm$ std over 100 images. \colorbox{blue!20}{\textbf{Bold}} indicates best per column.}
	\label{tab:celebahq}
\end{table*}

\paragraph{FID Scores for FFHQ Model}

We evaluate FID as a function of latent dimensionality $k$ within the $24\times 64\times 64$ latent space, computed against $10{,}000$ training images from FFHQ. The results are shown in \cref{fig:fid_ffhq}.

\begin{figure}[htbp]
	\centering
	\includegraphics[width=1.0\textwidth, height=0.4\linewidth, keepaspectratio]{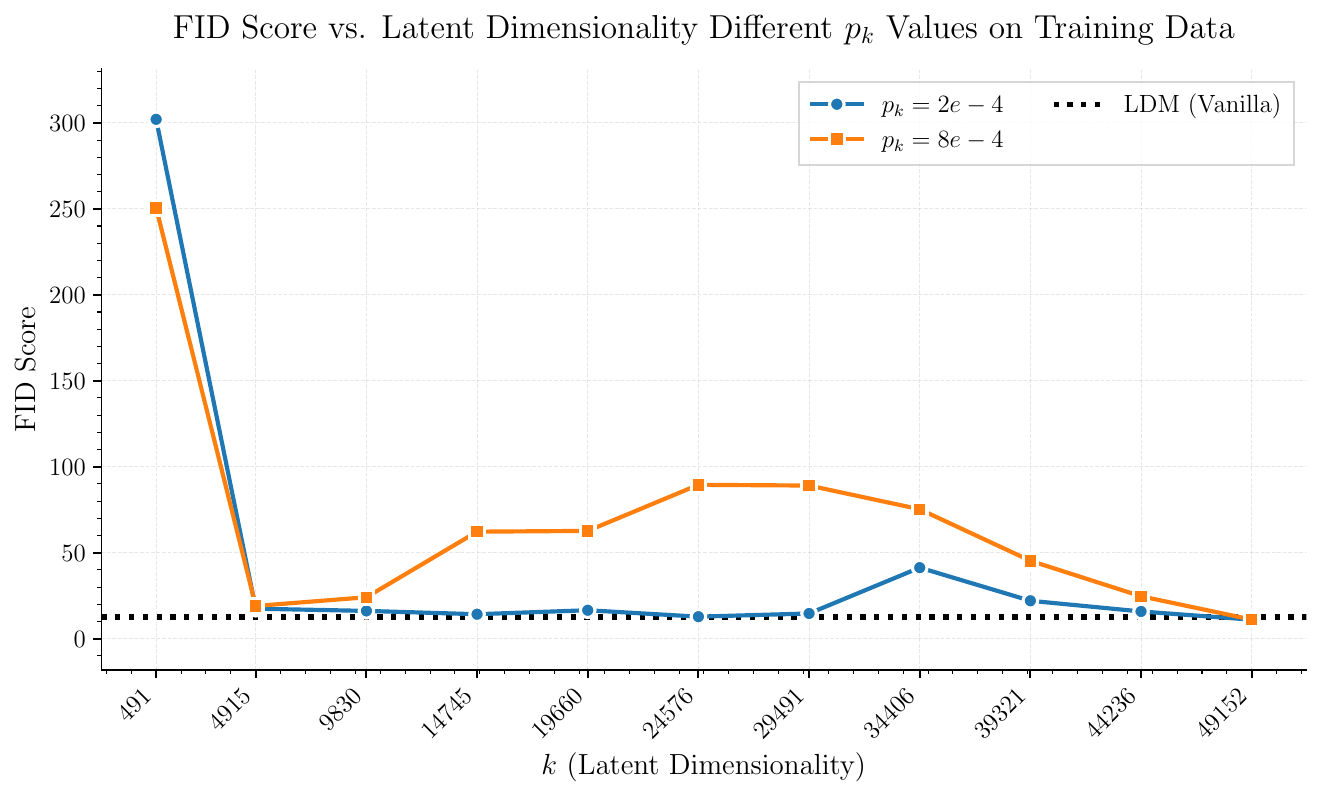}
	\caption{FFHQ FID score as a function of latent dimensionality $k$, computed against $10{,}000$ FFHQ training images.}
	\label{fig:fid_ffhq}
\end{figure}

\vspace{2em}
\section{Normalizing Flow Inverse Problems}
\label{app:NF}

We assume the practitioner has trained a tunable normalizing flow with the following objective provided by~\citep{bekasov2020ordering}. Our formulation is for latent-variable generative priors. For simplicity assume the measurements have the following form: the forward operator $\gA= \mA \in \R^{m \times n}$ and the additive noise term is i.i.d.\ $\bm{\eta} \sim \mathcal{N}(0,\sigma^2 \mI_{m})$, as in the compressed sensing case. Thus the measurements are generated by $\vy=\rmA \vx + \bm{\eta}$, where $\vx \in \R^{n}$ is an unknown signal.

Consider a fixed invertible neural network $G\colon\R^{n} \to \R^{n}$. Define for each $k \in [n] = \{1, \dots, n \}$ the map
\begin{align}
	\label{eqn:mapping}
	G_{k}  \colon \R^{k} & \longrightarrow \R^n \hspace{0.3cm}  \nonumber \\
	z                    & \longmapsto G\left(\zk\right). \nonumber
\end{align}
We note that $k$ is the parameter that governs the complexity of the prior, and  $\zk \in \R^{n}$ is obtained from the latent representation $\vz \in \R^{k}$. Then the signal representation is given by $\vx=G((\vz^T,0))$ where $\vx\in \R^{n}$ is as above.

We are interested in recovering a signal $\vx \in \R^{n}$ given a set of noisy measurements $\vy \in \R^{m}$. Our prior has a valid density over the entire signal space when $\vx \in \text{Range}(G_k)$, therefore a natural attempt to solve the given inverse problems is by a maximum \textsl{a posteriori} estimation:
\begin{align}
	\hat{\vx}_{\text{MAP}}(k) & \defeq  \argmin_{\vx \in \text{Range}(G_{k})}\; - \log p(\vx|\vy) \nonumber                                       \\
	                          & = \argmin_{\vx \in \text{Range}(G_{k})}\;-\log p(\vy|\vx) - \log p_{G_{k}}(\vx) \nonumber                         \\
	                          & = \argmin_{\vx \in \text{Range}(G_{k})}\; \frac{1}{2}\| \vy- \mA \vx \|^2 - \sigma^2 \log p_{G_{k}}(\vx) \nonumber
\end{align}
where $p_{G_{k}}$ is the density function on $\vx$ induced by $G_{k}$, $\| \cdot\|$ is the Euclidean norm, and $\sigma$ is given by the model noise. Similar to previous work of~\citep{asim2020invertible,whang2021solving}, we optimized over the latent space given by:
\begin{align}
	\hat{\vz}_{\text{MAP}} (k) & : =
	\argmin_{\vz \in \R^{k}}\; \frac{1}{2}\left\| \vy- \mA G \left(\zk\right) \right\|^2 \nonumber
	- \sigma^2 \log p_{G_{k}}\left(G\left(\zk\right)\right)  \nonumber
\end{align}

The proposed optimization problem above is solved via gradient descent with the Adam optimizer~\citep{KingBa15} and is initialized at $\vz=0$. Depending on the specific inverse problem at hand the tunable parameter $k$ can be selected to change the dimensionality of the model, so the first $k$ elements of $\vz$ will be optimized over and the rest will be set to zero. Our formulation relies on a density function $p_{G_{k}}$, and works by~\citep{asim2020invertible,whang2021solving} have shown that a smoothing parameter on a density function improves performance when solving the optimization problem above. We empirically observe the same phenomenon, and as a result, we replace $\sigma^2$ with a hyperparameter $\gamma$ and optimize the modified MAP estimate
\begin{align}
	\hat{\vz}_{\text{MAP}} (k) & : =
	\argmin_{\vz \in \R^{k}}\; \frac{1}{2}\left\| \vy- \mA G \left(\zk\right) \right\|^2 \nonumber
	- \gamma \log p_{G_{k}}\left(G\left(\zk\right)\right).  \nonumber
\end{align}

For the VAE we follow the same formulation as~\citep{bora2017compressed}, which is similar to the one above, but instead of having a tractable density, we approximate with an $\ell_{2}$ penalty in the latent space.

\begin{figure*}[t!]
	\centering
	\includegraphics[width=0.8\textwidth]{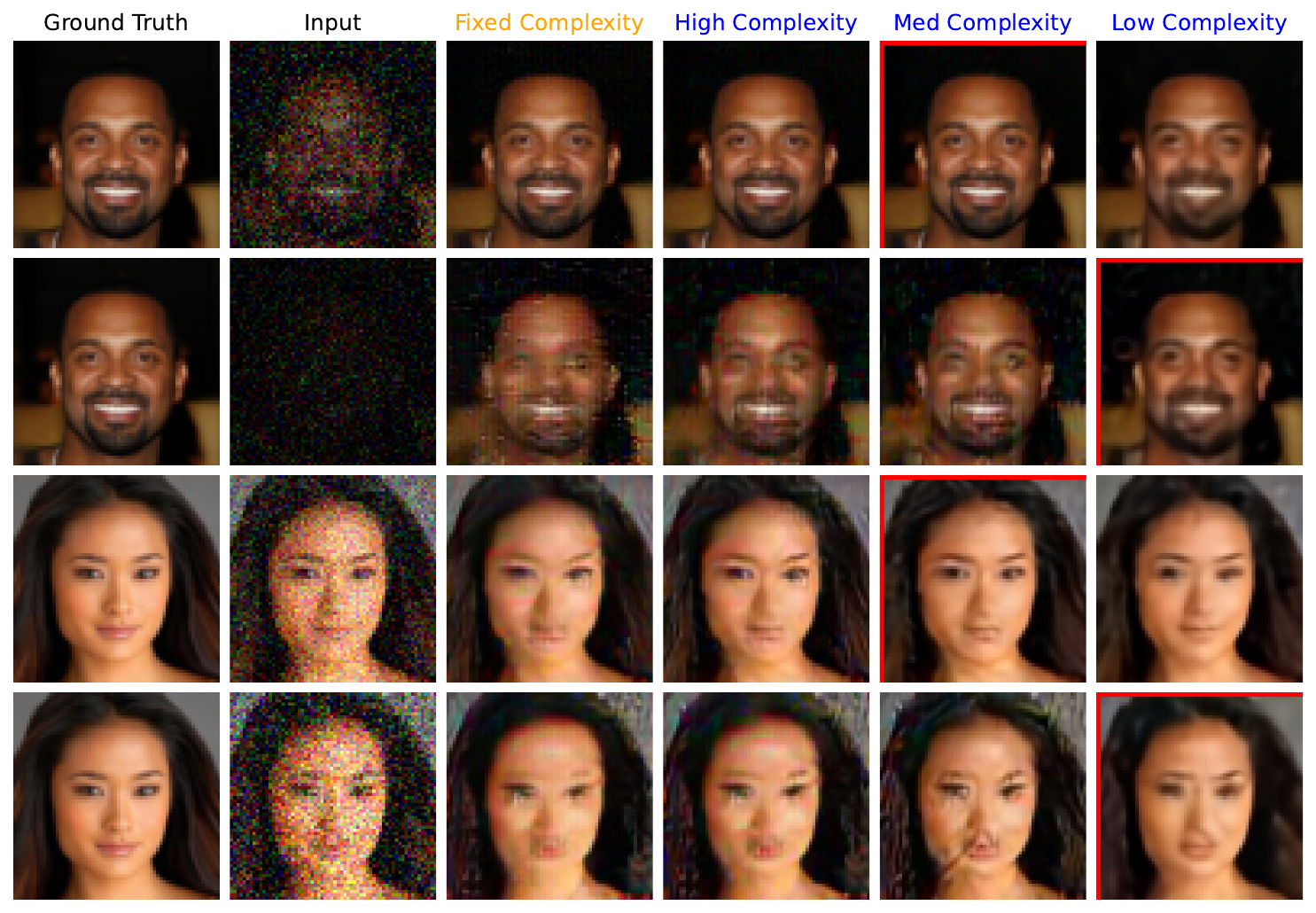} 
	\caption{Results on inpainting and denoising across various measurement regimes for Normalizing Flow. The tuning parameter should be chosen based on the complexities of the given inverse problem. The red bounding box indicates which model achieved the lowest reconstruction error with respect to SSIM.}
	\label{fig:tunable_CELBA_vis_nf}
\end{figure*}

\begin{figure}[H]

	\centering
	\includegraphics[width=1.0\textwidth]{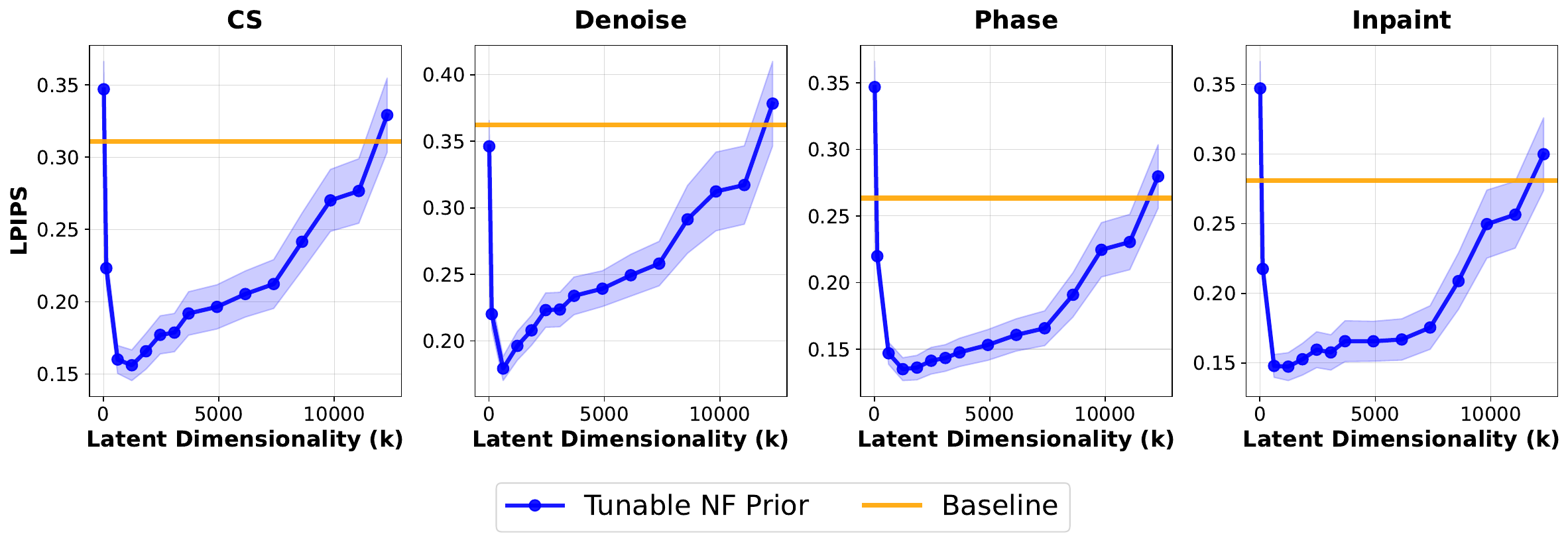}
	\caption{Performance of a generative prior with tunable complexity (Tunable NF Prior) and its fixed-complexity counterpart (Normalizing Flow) on compressed sensing, inpainting, phase retrieval, and denoising tasks on the CelebA dataset. The tunable prior exhibits a range of latent dimensionalities $k$ for which it yields improved reconstructions, as measured by the perceptual metric LPIPS, compared to the fixed-complexity baseline.
	}
	\label{fig:nf_flow_ushape}
\end{figure}

\paragraph{Generative Priors with Tunable Complexity on Various Architectures}
In these experiments we want to demonstrate that generative priors with tunable complexities are not tied to one specific normalizing-flow architecture. We train three architectures: Rational Quadratic Spline Flow~\citep{durkan2019neural}, GLOW~\citep{kingma2018glow}, and a GLOW-inspired variant called GLOW-ADD that uses an additive coupling layer instead of an affine one. Each architecture is trained on the SVHN dataset under matched hyperparameters to enable a controlled comparison.

\begin{figure}[h!]
	\centering
	\includegraphics[width=0.5\textwidth]{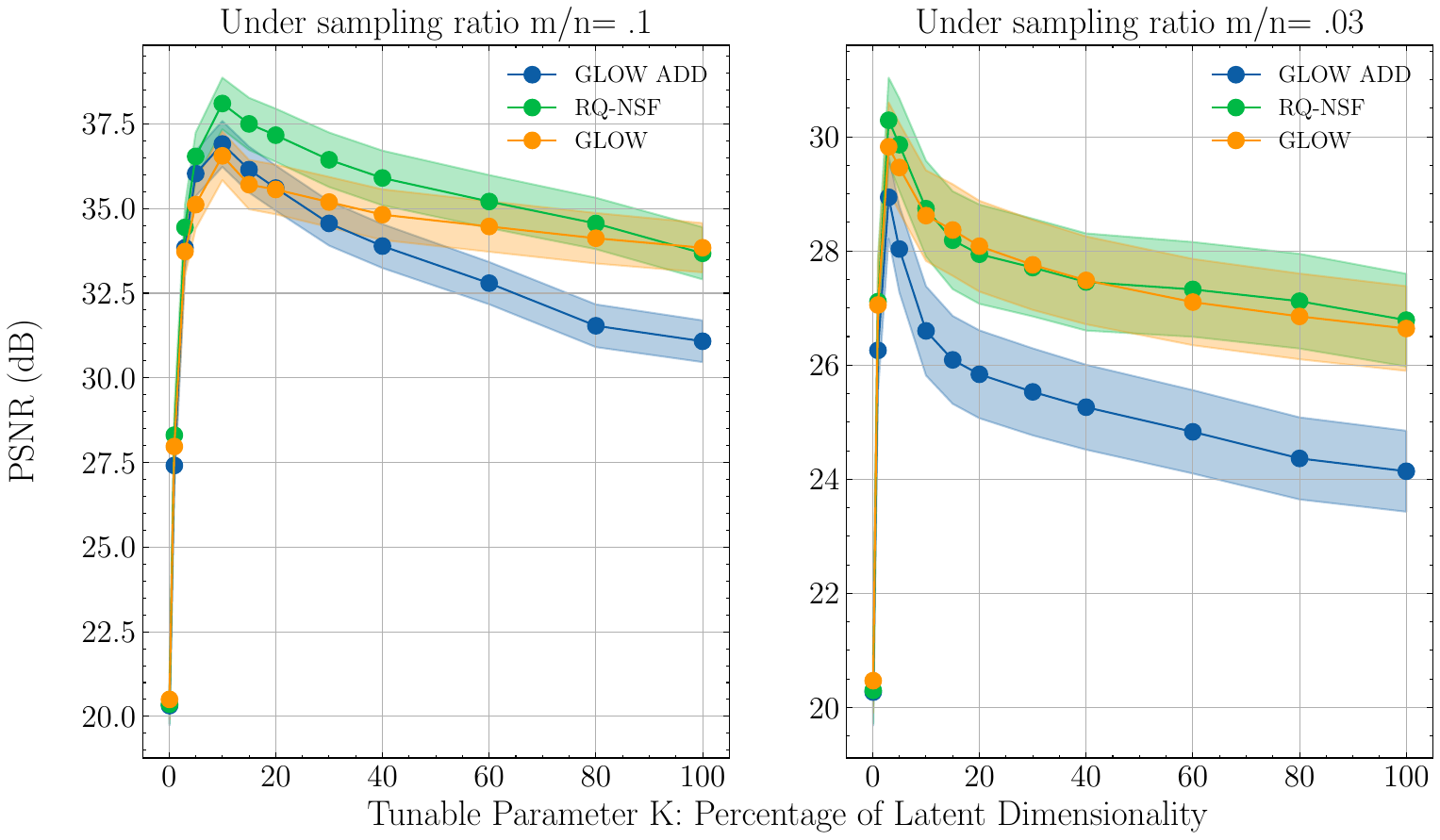}
	\caption{Reconstruction performance across three normalizing-flow architectures on SVHN; all benefit from a prior with tunable complexity.}
	\label{fig:SVHN}
\end{figure}

\begin{figure*}[h!]
	\centering
	\includegraphics[width=0.85\textwidth]{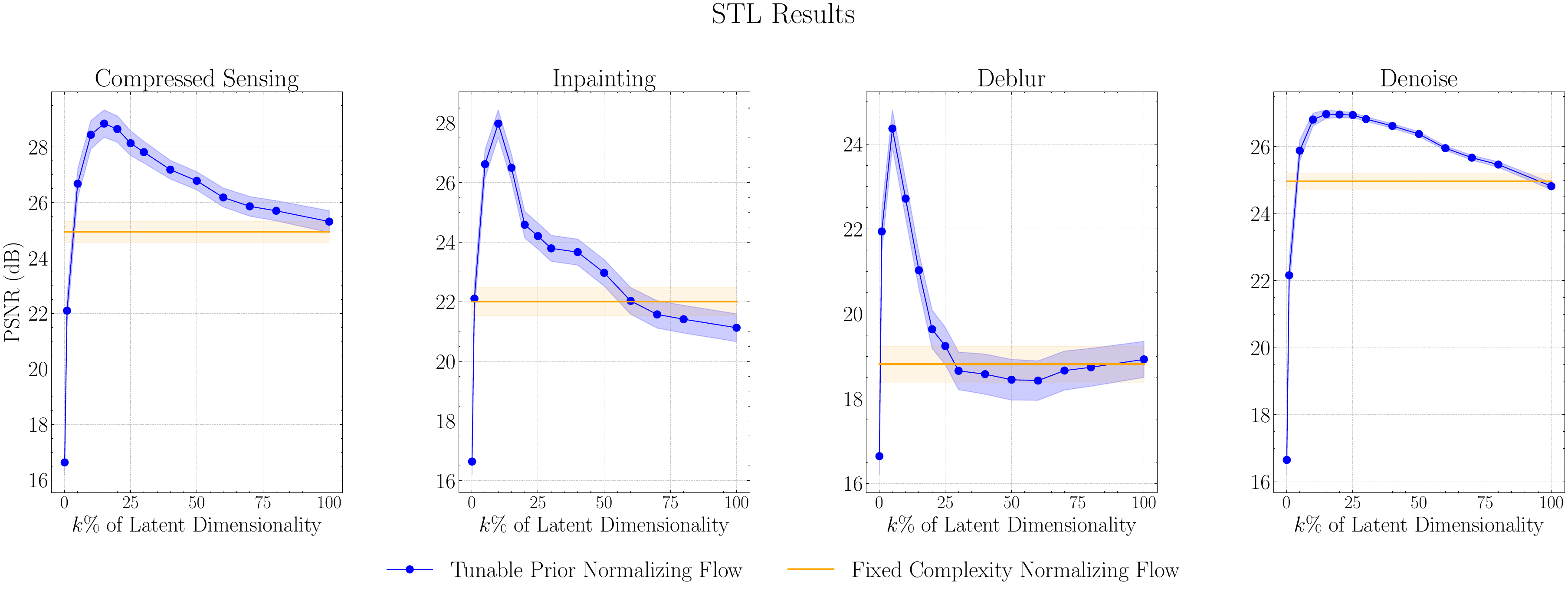}
	\caption{Performance of a generative prior with tunable complexity (Tunable Prior) and its fixed-complexity counterpart (Normalizing Flow) for compressed sensing, inpainting, deblurring, and denoising on the STL dataset. The tunable prior demonstrates a range of tunable parameters $k$ that lead to a better estimate of the target signal measured in PSNR than its baseline.}
	\label{fig:tunable_STL}
\end{figure*}

\begin{figure*}[h!]
	\centering
	\includegraphics[width=1\textwidth]{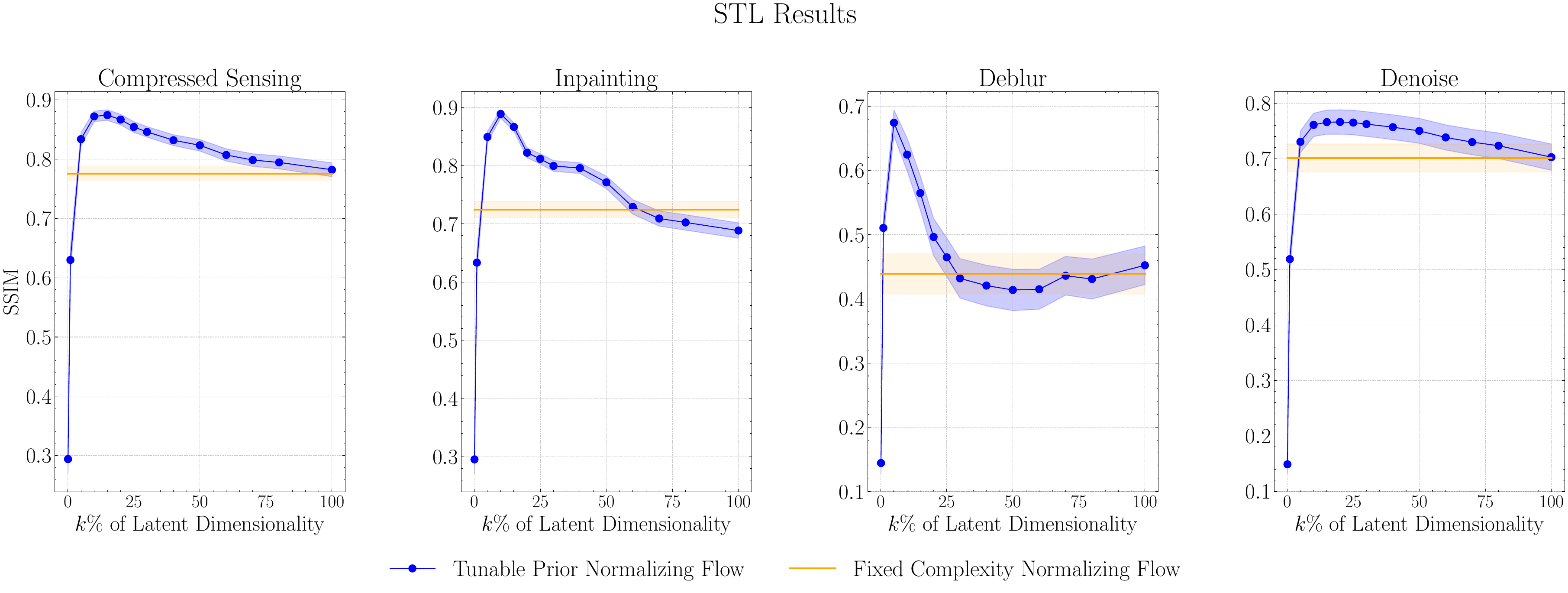}
	\caption{Performance of a generative prior with tunable complexity (Tunable Prior) and its fixed-complexity counterpart (Normalizing Flow) for compressed sensing, inpainting, deblurring, and denoising on the STL dataset. The tunable prior demonstrates a range of tunable parameters $k$, leading to a better estimate of the target signal measured in SSIM than its baseline.}
	\label{fig:tunable_STL_SSIM}
\end{figure*}

\newpage
\section{Proofs}
\label{app:proofs}
In \cref{sec:theory}, we present the maximum-likelihood setting in the main text because it is conceptually simpler. Here, we give the corresponding maximum \textsl{a posteriori} (MAP) formulation, of which maximum likelihood is a special case.

We consider a family of linear generative models defined as follows. Fix an invertible matrix $\mG \in \R^{n \times n}$, and let $\mG=\mU \bm{\Sigma} \mV^{T}$ be a singular value decomposition, where $\bm{\Sigma} = \diag(s_1, \ldots, s_n) \in \R^{n \times n}$. For any $k \leq n$, let $\Gk=\mU \Sk \mV^{T} \in \R^{n \times n}$, where
\[
	\Sk = \diag(s_1, \ldots, s_k, 0, \ldots, 0) \in \R^{n \times n}.
\]
Note that $\mG = \mG_n$. Each $\mG_k$ induces a probability distribution over $\R^n$ by
\[
	\vx = \mG_k \vz, \qquad \vz \sim \N(0, \mI_n).
\]
We denote this distribution by $p_{\mG_k}$ and observe that $p_{\mG_k} = \N(0, \mG_k \mG_k^T)$. The parameter $k$ controls the complexity of the modeled signal class.

We consider the following denoising problem. Let $\vy= \vx_{0} + \bm{\eta}$, where $\vx_{0} \sim p_{\Gn}$ and $\bm{\eta} \sim \N(0,\sigma^2 \mI_n)$ are independent. Our goal is to recover $\vx_0$ given $\vy$ and $\mG_n$. For a given $k$, we consider a maximum \textsl{a posteriori} estimate of $\vx_{0}$ under the signal prior $p_{\mG_k}$. This leads to the following optimization problem over the latent space:
\begin{align}
	\label{eq:latent_map}
	\zmapk
	 & = \argmin_{\vz^k \in \R^k}
	\frac{1}{2} \Bigl\|\vy-\Gk \zkt \Bigr\|^2
	+ \frac{\gamma}{2} \|\vz^k\|^2 .
\end{align}
Here, $\zkt \in \R^n$ denotes the vector $\vz^k$ padded with zeros, the estimated signal is $\xmapk = \mG_k \zmapkk$, and $\gamma \geq 0$ controls the strength of the latent prior. The case $\gamma=0$ corresponds to the maximum-likelihood estimate (MLE), while $\gamma=\sigma^2$ corresponds to the standard MAP estimator under the Gaussian latent prior. Since $\gamma$ is often treated as a hyperparameter in practice, we study the behavior of \cref{eq:latent_map} for all $\gamma \geq 0$.

The following theorem gives an exact expression for the mean-squared reconstruction error of this estimator.

\begin{theorem}
	\label{thm:map}
	Suppose we have a family $\{\Gk\}_{k=1}^{n}$ of generative models as defined above, with $p_{\Gk} = \N(0,\Gk \Gk^T)$. Let $\mG_n \in \R^{n \times n}$ have singular values $s_{1} \geq s_{2} \geq \cdots \geq s_{n} >0$, and suppose $\vx_{0} \sim p_{\mG_n}$ and $\bm{\eta} \sim \N(0,\sigma^2 \In)$ are independent. Then the estimator in \cref{eq:latent_map} satisfies
	\begin{align}
		\E_{\vx_0,\,\bm{\eta}}\big[\,\|\xmapk - \vx_{0}\|^2\,\big]
		= \sum_{i=1}^{k}
		\frac{s_{i}^2 \big(s_{i}^2 \sigma^2 + \gamma^2\big)}
		{(s_{i}^2 + \gamma)^2}
		+ \sum_{j=k+1}^{n} s_{j}^2 .
	\end{align}
\end{theorem}

The exact expression in \cref{thm:map} makes it possible to read off the optimal value of the tunable complexity parameter $k$, as established in the following corollary.

\begin{corollary}
	\label{cor:low_k_map}
	Under the assumptions of \cref{thm:map}, if $\gamma \leq \sigma^2/2$, then the value of $k$ that minimizes the reconstruction error is given by
	\begin{align}
		\argmin_{k \in [n]}
		\E_{\vx_0, \bm{\eta}} \big[ \| \xmapk- \vx_{0} \|^2 \big]
		= \max \Bigl\{ k \mid s_k \geq \sqrt{\sigma^2 - 2 \gamma}\Bigr\}.
	\end{align}
\end{corollary}

This corollary shows that, for sufficiently small $\gamma$ and sufficiently large noise, the optimal complexity parameter can be strictly smaller than the full signal dimensionality. In the MLE case $\gamma=0$, an intermediate signal complexity is optimal whenever the noise level exceeds the smallest singular value of the linear generator $\mG$. More generally, as the noise level increases, the optimal value of the tunable complexity parameter decreases. Setting $\gamma=0$ in \cref{thm:map,cor:low_k_map}, which \cref{thm:mse_k} demonstrates.

\begin{proof}
	[Proof of \cref{thm:mse_k}]
	\smallskip${}$
	\\
	Without loss of generality, we take $G_k$ to be a diagonal matrix.  This follows because of invariance of Euclidean norms with respect to orthogonal transformations and because of the rotational invariance of the random variable $\bm{z_0}$.
	Let $\Pk$ and $\Pko$ be the orthogonal projectors onto $\text{span}(\{e_1, \ldots, e_k\})$ and its orthogonal complement, respectively. Throughout this proof, we identify each projected vector with its coordinates in the corresponding standard basis, so that $\Pk(\vx) \in \R^k$ and $\Pko(\vx) \in \R^{n-k}$; this identification preserves Euclidean norms. For a diagonal matrix $\Sigma=\diag(s_{1},s_{2},\cdots,s_{n}) \in \R^{n \times n}$, let $\Skmat \in \R^{k \times k}$ and $\Snkmat \in \R^{n-k \times n-k}$ be such that
	\[
		\bm{\Sigma} = \begin{pmatrix}
			\Skmat         & \boldsymbol{0} \\
			\boldsymbol{0} & \Snkmat
		\end{pmatrix}.
	\]

	We want to find the expectation
	\begin{align}
		\E_{\vx_0, \eta} \| \xmapk-\vx_0 \|^2 & =
		\E_{\vx_0, \bm{\eta}} \| \Pk (\xmapk-\vx_0) \|^2 +  \E_{\vx_0} \| \Pko (\vx_0) \|^2.
		\label{eqn:exp-mse}
	\end{align}
	For the equation above we will compute the first term on the right hand side first, then do the same for the second term on the right hand side. The first term on the right hand side breaks into a sum of two terms which can readily be computed, while the second term on the right hand side follows from a lemma. Notice $\Pk(\xmapk) =\Skmat \zmapk \in \R^{k}$ and the solution of the optimization problem \eqref{eq:latent_map} is
	\begin{align*}
		\zmapk : & = (\Skmat \Skmat^{T} + \gamma \Ik)^{-1} \Skmat^{T} \Pk(y)                                \\
		         & = (\Skmat \Skmat^{T} + \gamma \Ik)^{-1} \Skmat^{T} (\Skmat \Pk(z_{0}) + \Pk(\bm{\eta})). \\
	\end{align*}
	Thus,
	\begin{align}
		\E_{\vx_0, \bm{\eta}} \| & \Pk (\xmapk-\vx_0) \|^2 = \E_{\bm{z_0}, \bm{\eta}} \| \Skmat \zmapk- \Skmat \Pk(\bm{z_0}) \|^2 \nonumber                          \\
		                         & = \E_{\bm{z_0}, \bm{\eta}} \| \Skmat (\Skmat \Skmat^{T} + \gamma \Ik)^{-1} \Skmat^{T} \Pk(y)- \Skmat \Pk(\bm{z_0}) \|^2 \nonumber \\
		                         & = \E_{\bm{z_0}, \bm{\eta}} \| \Skmat (\Skmat \Skmat^{T} + \gamma \Ik)^{-1} \Skmat^{T} (\Skmat \Pk(z_{0}) + \Pk(\bm{\eta})) \notag \\
		                         & - \Skmat \Pk(\bm{z_0}) \|^2 \nonumber                                                                                             \\
		                         & = \E_{\bm{z_0}, \bm{\eta}} \| (\Skmat (\Skmat \Skmat^{T} + \gamma \Ik)^{-1} \Skmat^{T} \Skmat -\Skmat) \Pk(\bm{z_0}) \notag \\
		                         & + \Skmat (\Skmat \Skmat^{T} + \gamma \Ik)^{-1} \Skmat^{T} \Pk(\bm{\eta}) \|^2 \label{eq:decompose}
	\end{align}
	The last term above \cref{eq:decompose} can be decomposed
	\begin{align}
		 & \E_{\bm{z_0}} \left\| \left( \Skmat (\Skmat \Skmat^{T} + \gamma \Ik)^{-1} \Skmat^{T} \Skmat - \Skmat \right) \Pk(\vz_0) \right\|^2 \label{eq:A} \\
		 & + 2\E_{\bm{z_0}, \bm{\eta}} \left\langle \left( \Skmat (\Skmat \Skmat^{T} + \gamma \Ik)^{-1} \Skmat^{T} \Skmat - \Skmat \right) \Pk(\bm{z_0}), \Skmat (\Skmat \Skmat^{T} + \gamma \Ik)^{-1} \Skmat^{T} \Pk(\bm{\eta}) \right\rangle \label{eq:inner} \\
		 & + \E_{\bm{\eta}} \left\| \Skmat (\Skmat \Skmat^{T} + \gamma \Ik)^{-1} \Skmat^{T} \Pk(\bm{\eta}) \right\|^2 \label{eq:B}
	\end{align}

	The expectation of the inner product term (\cref{eq:inner}) is zero because $\vz_0$ and $\bm{\eta}$ are independent and have zero means. Furthermore, we apply \cref{lemma:gauss_fro} to \cref{eq:A}, yielding

	\begin{align}
		\E_{\bm{z_0}} & \| (\Skmat (\Skmat \Skmat^{T} + \gamma \Ik)^{-1} \Skmat^{T} \Skmat - \Skmat) \Pk(\vz_0) \|^{2} \nonumber \\
		              & = \| \Skmat (\Skmat \Skmat^{T} + \gamma \Ik)^{-1} \Skmat^{T} \Skmat - \Skmat\|_{F}^{2} \nonumber \\
		              & = \sum_{i=1}^{k} \Bigl| \frac{s_{i}^3}{ s_{i}^2 + \gamma } - s_{i} \Bigr|^2 \nonumber                    \\
		              & = \sum_{i=1}^{k} \frac{s_{i}^2 \gamma^2}{(s_{i}^2 + \gamma)^2} \label{eq:exp_final_k}.
	\end{align}
	Then we apply \cref{lemma:gauss_l2_vec} to \cref{eq:B}
	\begin{align}
		\E_{\bm{\eta}} & \|\Skmat (\Skmat \Skmat^{T} + \gamma \Ik)^{-1} \Skmat^{T} \Pk(\bm{\eta}) \|^2 \nonumber                                                       \\
		               & = \Tr(\Skmat (\Skmat \Skmat^{T} + \gamma \Ik)^{-1} \Skmat^{T} \sigma^2 \Ik \Skmat (\Skmat \Skmat^{T} + \gamma \Ik)^{-1} \Skmat^{T}) \nonumber \\
		               & = \sum_{i=1}^{k} \frac{s_{i}^4 \sigma^2}{ (s_{i}^2 + \gamma)^2} \label{eq:exp_final_noise}.
	\end{align}
	This concludes the calculation of the first term on the right hand side of \cref{eqn:exp-mse}. Now we apply \cref{lemma:gauss_fro} to the second term on the right hand side
	\begin{align}
		\label{eqn:ortho_exp}
		\E_{\vx_0 }  \| \Pko (\vx_0) \|^2  = \E_{z_{0}} \| \Snkmat \Pko(z_{0}) \|^2 = \| \Snkmat \|_{F}^2 = \sum_{j=k+1}^n s_{j}^2.
	\end{align}

	Lastly, we combine Eq.\ref{eqn:ortho_exp},\ref{eq:exp_final_k},\ref{eq:exp_final_noise} and simplify
	\[
		\E_{\vx_0, \bm{\eta}} \| \xmapk-\vx_0 \|^2
		= \sum_{i=1}^{k} \frac{s_{i}^2 (s_{i}^2 \sigma^2 + \gamma^2)}{(s_{i}^2 + \gamma)^2} + \sum_{j=k+1}^n s_{j}^2.
	\]
\end{proof}

\begin{proof}
	[Proof of Corollary~\ref{cor:low_k}]
	\smallskip${}$
	\\
	For notational convenience, let $E(k) = \E_{\vx_0, \bm{\eta}} \| \hat{\vx}_\gamma (k) - \vx_0 \|^2$. Since $[n]$ is a finite non-empty set, $\argmin_{k} \, E(k)$ exists, and so we consider the set $C$ of \emph{candidate} minimizers,
	\[
		C = \{k \mid E(k - 1) \geq E(k)\}.
	\]
	Notice $E(k-1) - E(k) \geq 0$ for each $k \in C$, and moreover that this difference is
	\begin{align*}
		 & \sum_{i=1}^{k-1}  \frac{s_{i}^2 (s_{i}^2 \sigma^2 + \gamma^2)}{(s_{i}^2 + \gamma)^2} + \sum_{j=k}^n s_{j}^2 - \sum_{i=1}^{k} \frac{s_{i}^2 (s_{i}^2 \sigma^2 + \gamma^2)}{(s_{i}^2 + \gamma)^2} \notag \\
		 & - \sum_{j=k+1}^n s_{j}^2 \geq 0                                                                                                                                                                        \\
		 & \iff                                                                                                                                                                                                   \\
		 & s_{k}^2 - \frac{s_{k}^2 (s_{k}^2 \sigma^2 + \gamma^2)} {(s_{k}^2 + \gamma)^2} \geq 0 \\
		 & \iff                                                                                                                                                                                                   \\
		 & s_{k} \geq \sqrt{\sigma^2 - 2 \gamma},
	\end{align*}
	which reveals
	\[
		C = \{k \mid s_k \geq \sqrt{\sigma^2 - 2\gamma}\}.
	\]
	From our assumptions, $\sigma^2 - 2\gamma \geq 0$. Note now that for any $i, j \in C$ with $i < j$, we have that
	\begin{align*}
		s_i \geq s_k \geq s_j \geq \sqrt{\sigma^2 - 2\gamma} \quad \text{for all $i < k \leq j$,}
	\end{align*}
	hence $k \in C$.

	To prove the corollary is to prove that $\argmin_k E(k) =\max C$, so it suffices to show $E$ is decreasing on $C$. Since $s_i^2 \geq \sigma^2 - 2\gamma$ for all $i \in C$, it is clear that
	\begin{align*}
		s_i^4 + s_i^2(2\gamma - \sigma^2) \geq s_i^4 - s_i^2 \cdot s_i^2 = 0
	\end{align*}
	is non-negative. From this, we deduce that $s_i^2 \sigma^2 + \gamma^2 \leq (s_i^2 + \gamma)^2$, hence
	\begin{align*}
		\Delta(i) \defeq  1 - \frac{s_i^2 \sigma^2 + \gamma^2}{(s_i^2 + \gamma)^2} \geq 0
	\end{align*}
	is also non-negative.

	Let $P, Q \in C$ be such that $P \leq Q$. We'll show $E(P) \geq E(Q)$ by showing the difference $E(P) - E(Q)$ is non-negative. It is easily computed by
	\begin{align*}
		E(P) - E(Q) = \sum_{i = P + 1}^Q s_i^2 \Delta(i).
	\end{align*}
	The above shows that $\Delta(i)$ is non-negative for all $i$ in the range $P < i \leq Q$ since all such $i$ are elements of $C$, thus concluding the proof.
\end{proof}

\begin{lemma}\label{lemma:gauss_fro}
	Let $\vx \in \R^n$ satisfy $\vx \sim \mathcal{N}(0,\In)$. Then, for any matrix $\mM \in \R^{m \times n}$, we have
	\begin{align*}
		\E_{\vx} \Vert \mM \vx\Vert^2 = \|\mM \|_{F}^2.
	\end{align*}
\end{lemma}

\begin{proof}[Proof of Lemma~\ref{lemma:gauss_fro}]
	\smallskip${}$
	\\
	For matrices, let $\langle A,B\rangle_F = \Tr(A^T B)$ denote the Frobenius inner product. Since $\E_{\vx}[\vx\vx^T]=\In$, we have
	\begin{align*}
		\E_{\vx} \| \mM \vx \|^2 & =\E_{\vx} \langle \mM \vx, \mM \vx\rangle                \\
		                         & = \E_{\vx} \langle \mM^{T} \mM,\vx \vx^{T}\rangle_F \notag \\
		                         & =\langle \mM^{T} \mM,\In \rangle_F \notag                  \\
		                         & = \| \mM \|_{F}^2.
	\end{align*}
\end{proof}

\begin{lemma}\label{lemma:gauss_l2_vec}
	Let $\vx \in \R^n$ satisfy $\vx \sim \mathcal{N}(0,\bm{\Sigma})$. Then, for any matrix $\mM \in \R^{m \times n}$, we have
	\begin{align*}
		\E_{\vx} \Vert \mM \vx\Vert^2 = \Tr(\mM \bm{\Sigma} \mM^T).
	\end{align*}
\end{lemma}

\begin{proof}[Proof of Lemma~\ref{lemma:gauss_l2_vec}]
	\smallskip${}$
	\\
	Let $\mY=\mM\vx \in \R^m$. Then $\mY \sim \mathcal{N}(0,\mM \bm{\Sigma} \mM^T)$. Thus,
	\begin{align}
		 & \E_{\vx} \Vert \mM \vx\Vert^2 = \E_{\mY} [\mY^T \mY] \notag \\
		 & = \sum_{i=1}^m \E_{\mY} [\mY_{i}^2] = \sum_{i=1}^m \text{Var}(\mY_{i}) \notag \\
		 & = \Tr(\mM \bm{\Sigma} \mM^T).
	\end{align}

\end{proof}

\end{document}